\documentclass{article}

\PassOptionsToPackage{numbers}{natbib}

\usepackage[final]{neurips_2020}

\usepackage[utf8]{inputenc} %
\usepackage[T1]{fontenc}    %
\usepackage{hyperref}       %
\usepackage{url}            %
\usepackage{booktabs}       %
\usepackage{amsfonts}       %
\usepackage{nicefrac}       %
\usepackage{microtype}      %

\usepackage{graphicx}
\usepackage{lipsum}
\usepackage{xcolor}
\usepackage{amsthm}
\usepackage{amsmath}
\usepackage{amssymb}
\usepackage{bm}
\usepackage{tikz}
\usetikzlibrary{arrows.meta}
\usepackage{multirow}
\usepackage{subcaption}
\usepackage{thmtools}
\usepackage{thm-restate}
\usepackage{chngcntr}
\usepackage{fancyvrb}
\usepackage{environ}
\usepackage{moresize}

\let\emptyset\varnothing
\declaretheorem[name=Proposition]{prop}
\declaretheorem[name=Corollary,numbered=no]{corollary*}
\newcommand{\R}{\mathbb{R}}
\newcommand{\scsym}[1]{\textsc{\small#1}}
\newcommand{\bvec}[1]{\bm{#1}}
\newcommand{\E}{\mathbb{E}}

\DeclareMathOperator*{\argmax}{argmax}

\DeclareMathOperator{\softmax}{softmax}

\title{Learning Graph Structure With A Finite-State Automaton Layer}

\makeatletter
\let\@fnsymbol\@arabic
\makeatother
\author{%
  Daniel D. Johnson,\, Hugo Larochelle,\, Daniel Tarlow \\
  Google Research\\
  \texttt{\{ddjohnson, hugolarochelle, dtarlow\}@google.com} \\
}

\begin{document}

\maketitle

\begin{abstract}
Graph-based neural network models are producing strong results in a number of domains, in part because graphs provide flexibility to encode domain knowledge in the form of relational structure (edges) between nodes in the graph. In practice, edges are used both to represent intrinsic structure (e.g., abstract syntax trees of programs) and more abstract relations that aid reasoning for a downstream task (e.g., results of relevant program analyses). In this work, we study the problem of learning to derive abstract relations from the intrinsic graph structure. Motivated by their power in program analyses, we consider relations defined by paths on the base graph accepted by a finite-state automaton. We show how to learn these relations end-to-end by relaxing the problem into learning finite-state automata policies on a graph-based POMDP and then training these policies using implicit differentiation. The result is a differentiable Graph Finite-State Automaton (GFSA) layer that adds a new edge type (expressed as a weighted adjacency matrix) to a base graph. We demonstrate that this layer can find shortcuts in grid-world graphs and reproduce simple static analyses on Python programs. Additionally, we combine the GFSA layer with a larger graph-based model trained end-to-end on the variable misuse program understanding task, and find that using the GFSA layer leads to better performance than using hand-engineered semantic edges or other baseline methods for adding learned edge types.
\end{abstract}

\section{Introduction}
Determining exactly which relationships to include when representing an object as a graph is not always straightforward. As a motivating example, consider a dataset of source code samples. One natural way to represent these as graphs is to use the abstract syntax tree (AST), a parsed version of the code where each node represents a logical component.\footnote{For instance, the AST for \texttt{print(x + y)} contains nodes for \texttt{print}, \texttt{x}, \texttt{y}, \texttt{x + y}, and the call as a whole.} But one can also add additional edges to each graph in order to better capture program behaviors. Indeed, adding additional edges to represent control flow or data dependence has been shown to improve performance on code-understanding tasks when compared to a AST-only or token-sequence representation \citep{allamanis2018learning,hellendoorn2020global}.

An interesting observation is that these additional edges are fully determined by the AST, generally by using hand-coded static analysis algorithms. 
This kind of program abstraction is reminiscent of temporal abstraction in reinforcement learning (e.g., action repeats or options \cite{mnih2013playing,sutton1999between}). In both cases, derived higher-level relationships allow reasoning more abstractly and over longer distances (in program locations or time).

In this work, we construct a differentiable neural network layer by combining two ideas: program analyses expressed as reachability problems on graphs \citep{reps1998program}, and mathematical tools for analyzing temporal behaviors of reinforcement learning policies \citep{dayan1993improving}.
This layer, which we call a Graph Finite-State Automaton (GFSA), can be trained end-to-end to add derived relationships (edges) to arbitrary graph-structured data based on performance on a downstream task.\footnote{An implementation is available at \url{https://github.com/google-research/google-research/tree/master/gfsa}.} We show empirically that the GFSA layer has favorable inductive biases relative to baseline methods for learning edge structures in graph-based neural networks.

\section{Background}

\subsection{Neural Networks on Graphs}\label{bkgd-graph-models}
Many neural architectures have been proposed for graph-structured data. We focus on two general families of models: first, message-passing neural networks (MPNNs), which compute sums of messages sent across each edge \citep{gilmer2017neural}, including recurrent models such as Gated Graph Neural Networks (GGNNs) \cite{li2015gated}; second, transformer-like models operating on the nodes of a graph, which include the Relation-Aware Transformer (RAT) \citep{wang2019rat} and Graph Relational Embedding Attention Transformer (GREAT) \citep{hellendoorn2020global} models along with other generalizations of relative attention \citep{shaw2018self}. All of these models assume that each node is associated with a feature vector, and each edge is associated with a feature vector or a discrete type.

\begin{figure}[t]
\centering
\begin{subfigure}[b]{.3\textwidth}
\centering
\includegraphics[height=1.1\textwidth]{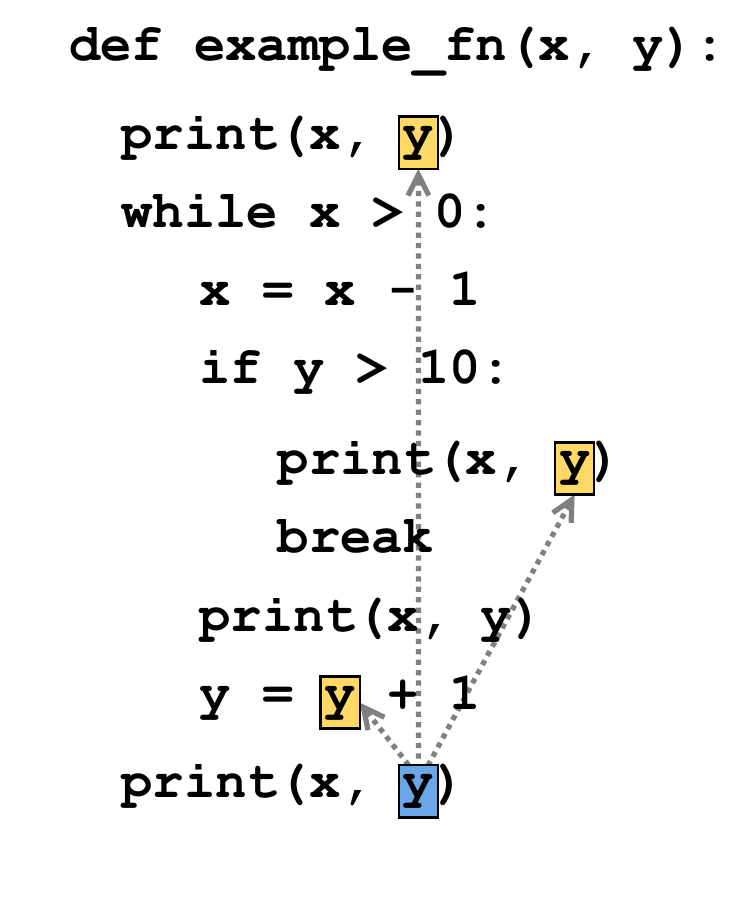}
\caption{\scsym{LastRead} edges for an \\\phantom{(a) }example function.}
\label{fig:lastread-truth}
\end{subfigure}
\begin{subfigure}[b]{.3\textwidth}
\centering
\includegraphics[height=1.1\textwidth]{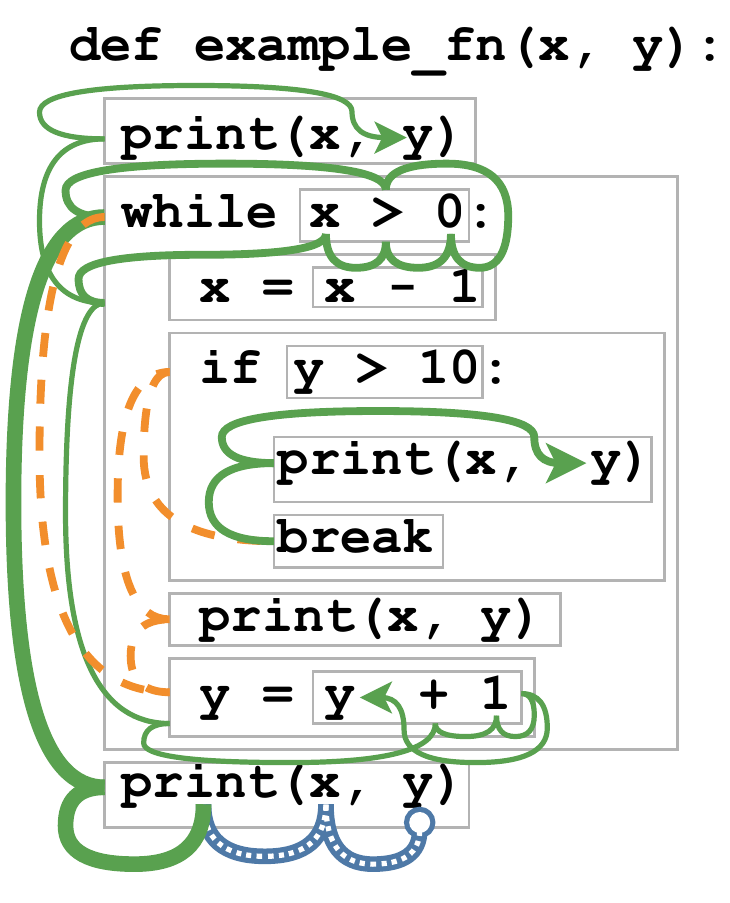}
\caption{Learned behavior of GFSA\\\phantom{(b) }on \scsym{LastRead} task.}
\label{fig:lastread}
\end{subfigure}
\begin{subfigure}[b]{.3\textwidth}
\begin{minipage}[b][1.1\textwidth]{\textwidth}
\vspace*{\fill}
\centering
\definecolor{tab10c0}{rgb}{0.12156862745098039,0.4666666666666667,0.7058823529411765}
\definecolor{tab10c1}{rgb}{1.0,0.4980392156862745,0.054901960784313725}
\definecolor{tab10c2}{rgb}{0.17254901960784313,0.6274509803921569,0.17254901960784313}
\definecolor{tab10c3}{rgb}{0.8392156862745098,0.15294117647058825,0.1568627450980392}
\begin{tikzpicture}[
x=.24cm,
y=.24cm,
pathline/.style={line width=0.45mm,->},
mazecell/.style={fill=black,draw=black!80!white,thick},
pathstart/.style={fill=white},
option0/.style={draw=tab10c0,-{Straight Barb [length=1.7mm, width=1.7mm]}},
option1/.style={draw=tab10c1,-{Triangle [length=1.5mm, width=1.5mm]}},
option2/.style={draw=tab10c2,-{Stealth [length=1.9mm, width=1.9mm]}},
option3/.style={draw=tab10c3,-{Triangle [open, length=1.7mm, width=1.7mm]}}
]
\pgfmathsetmacro{\s}{0.12}
\input{figures/maze_tikz.tex}
\end{tikzpicture}%
\vspace*{\fill}
\end{minipage}
\caption{Learned behavior of GFSA\\\phantom{(c) }on grid-world task.}
\label{fig:gridworld}
\end{subfigure}%
\caption{
(a) Target edges for the \scsym{LastRead} task starting from the final use of \texttt{y} on a handwritten example function.
(b) Learned behavior starting at the final use of \texttt{y} (blue circle). Thickness represents probability mass, color and style represent the finite-state memory, and boxes represent AST nodes in the graph. The automaton changes to a reverse-execution mode (green) and steps backward to the while loop, then nondeterministically either looks at the condition or switches to a break-finding mode (orange) and jumps to the body. In the first case, it checks for uses of \texttt{y} in the condition, then splits again between the previous print and the loop body. In the second, it walks upward until finding a break statement, then transitions back to the reverse-execution mode. For simplicity, we hide backtracking trajectories and combine some intermediate steps. Note that only the start and end locations (colored boxes in (a)) are supervised; all intermediate steps are learned.
(c) Colored arrows denote the path taken by the GFSA policy for each option, shown starting from four arbitrary start positions (white) on a grid-world layout not seen during training. The tabular agent can jump from each start position to the endpoint of any of its arrows in a single step.
}
\label{fig:gridworld-lastread}
\end{figure}

\subsection{Derived Relationships as Constrained Reachability}\label{subsec:derived-relationships-as-constrained-reachability}
Compilers and static analysis tools use a variety of techniques to analyze programs, many of which are based on fixed-point analysis over a problem-specific abstract lattice (see for instance \citet{cousot1977abstract}). However, it is possible to recast many of these analyses within a different framework: graph reachability under formal language constraints \citep{reps1998program}. 

Consider a directed graph $G$ where the nodes and edges are annotated with labels from finite sets $\mathcal{N}$ and $\mathcal{E}$. Let $L$ be a formal language over alphabet $\Sigma = \mathcal{N} \cup \mathcal{E}$, i.e. $L \subseteq \Sigma^*$ is a set of words (finite sequences of labels from $\Sigma$) built using formal rules.
One useful family of languages is the set of \emph{regular languages}: a regular language $L$ consists of the words that match a regular expression, or equivalently the words that a finite-state automaton (FSA) accepts \citep{hopcroft2001introduction}.
Note that each path in $G$ corresponds to a word in $\Sigma^*$, obtained by concatenating the node and edge labels along the path.
We say that a path from node $n_1$ to $n_2$ is an $L$-path if this word is in $L$.

Using a construction similar to \citet{reps1998program}, one can construct a regular language $L$ such that,
if $n_1$ is the location of a variable and $n_2$ is the location of a previous read from that variable, then there is an $L$-path from $n_1$ to $n_2$ (and every $L$-path is of this form)%
; roughly, $L$ contains paths that trace backward through the program's execution. This
corresponds to edge type \scsym{LastRead} as described by \citet{allamanis2018learning}, which is visualized in Figure \ref{fig:lastread-truth}. Similarly, one can define edges and corresponding regular languages that connect each use of a variable to the possible locations of the previous assignment to it (the \scsym{LastWrite} edge type from \citep{allamanis2018learning})
or connect each statement to the statements that could execute directly afterward (which we denote \scsym{NextControlFlow} since it corresponds to the control flow graph); see appendix \ref{appendix:prog-analysis-regular} for details.
More generally, the existence of $L$-paths summarizes the presence of longer chains of relationships with a single pairwise relation. Depending on $L$, this can represent transitive closures, compositions of edges, or other more complex patterns. We may not always know which language $L$ would be useful for a task of interest; our approach makes it possible to jointly learn $L$ and use the $L$-paths as an abstraction for the downstream task.

\section{Approach}
Consider, as a motivating example, a graph representing the abstract syntax tree for a Python program. Each of the nodes of this tree has a type, for instance ``identifier'', ``binary operation'' or ``if statement'', and edges correspond to AST fields such as ``X is the left child of Y'' or ``X is the loop condition of Y''. Section \ref{subsec:derived-relationships-as-constrained-reachability} suggests that $L$-paths on this graph are a useful abstraction for higher-level reasoning, but we do not know what the best choice of $L$ is; we seek a mechanism to learn it end-to-end.

We propose placing an agent on a node of this tree, with actions corresponding to the possible fields of each node (e.g. ``go to parent node'' or ``go to left child''), and observations giving local information about each node (e.g. ``this is an identifier, but not the one we are trying to analyze''). Note that the trajectories of this agent then correspond to paths in the graph.
We allow the agent to terminate the episode and add an edge from its initial to its current location, thus ``accepting" the path it has taken. By averaging over all trajectories, we obtain an expected adjacency matrix for these edges
that summarizes the paths that the agent tends to accept,
which we use as an output edge type.

If the agent's actions were determined by a finite-state automaton for a regular language $L$, the added edges would correspond to $L$-paths. We propose parameterizing the agent with a \emph{learnable} finite-state automaton, so that it can learn to do the kinds of analyses that a regular language can express. As long as the actions and observations are shared across all ASTs, we can then apply this policy to many different ASTs, even ones not seen at training time.

In this section, we formalize and generalize this intuition by describing a transformation from graphs into partially-observable Markov decision processes (POMDPs). We show that, for agents with a finite-state hidden memory, we can efficiently compute and differentiate through the distribution of trajectory endpoints. We propose using this distribution to define a new edge type, and demonstrate that any regular-language-constrained reachability problem (and in particular, basic program analyses) can be expressed as a policy of this form.

\subsection{From Graphs to POMDPs}\label{subsec:from-graphs-to-pomdps}

Suppose we have a family of graphs $\mathcal{G}$ with an associated set of node types $\mathcal{N}$.
Our approach is to transform each graph $G \in \mathcal{G}$ to a 
rewardless POMDP, in which an agent takes a sequence of actions to move between nodes of the graph while observing only local information about its current location.
To ensure that all graphs have compatible action and observation spaces, for each node type $\tau(n) \in \mathcal{N}$ we choose a finite set $\mathcal{M}_{\tau(n)}$ of movement actions associated with that node type (e.g. the set of possible fields that can be followed) and a finite set $\Omega_{\tau(n)}$ of observations (which include the node type as well as other task-specific information). These choices may depend on domain knowledge about the graph family or the task to be solved; see appendix \ref{appendix:encoding-graphs-as-pomdps} for the specific choices we used in our experiments.

At each node $n_t \in N$ of a graph $G \in \mathcal{G}$, the agent selects an action $a_t$ from the set
\[
\mathcal{A}_t = \left\{ (\scsym{Move}, m) ~\middle|~ m \in \mathcal{M}_{\tau(n_t)} \right\} \sqcup \left\{ \scsym{AddEdgeAndStop},\, \scsym{Stop},\, \scsym{Backtrack} \right\}
\]
according to some policy $\pi$. If $a_t = \scsym{AddEdgeAndStop}$, the episode terminates by adding an edge $(n_0, n_t)$ to the output adjacency matrix. If $a_t = \scsym{Stop}$, the episode terminates without adding an edge. If $a_t = (\scsym{Move}, m)$, the agent is either moved to an adjacent node $n_{t+1} \in N$ in the graph by the environment or stays at node $n_t$ (and thus $n_{t+1} = n_t$), and then receives an observation $\omega_{t+1}$ with $\omega_{t+1} \in \Omega_{\tau(n_{t+1})}$.
The MDP is partially observable because the agent does not see node identities or the global structure; instead, $\omega_{t+1}$ encodes only node types and other local information.
Since derived edge types may depend on existing pairwise relationships between nodes (for instance, whether two variables have the same name), we allow the observations to depend on the initial node $n_0$ as well as the current node $n_t$ and most recent transition; in effect, each choice of $n_0 \in N$ specifies a different version of the POMDP for graph $G$. Finally, if $a_t = \scsym{Backtrack}$ is selected, the agent is reset to its initial state. We note that, since the action and observation spaces are shared between all graphs in $\mathcal{G}$, a single policy $\pi$ can be applied to any graph $G \in \mathcal{G}$.

We would like our agent to be powerful enough to extract useful information from this POMDP, but simple enough that we can efficiently compute and differentiate through the learned trajectories. Since our motivating program analyses can be represented as regular languages, which correspond to finite-state automata (see section \ref{subsec:derived-relationships-as-constrained-reachability}), we focus on agents augmented with a finite-state memory.

\subsection{Computing Absorbing Probabilities}\label{subsec:computing-absorbing}
Here we describe an efficient way to compute and differentiate through the distribution over trajectory endpoints for a finite-state memory policy over a graph.
Let $Z$ be a finite set of memory states, and consider a specific policy $\pi_\theta(a_t, z_{t+1} ~|~ \omega_t, z_t)$ parameterized by $\theta$ (see appendix \ref{appendix:subsec-gfsa-impl-parameters}
for details regarding the parameterization we use). Combining the policy $\pi_\theta$ with the environment dynamics for a single graph $G$ yields an absorbing Markov chain over tuples $(n_{t}, \omega_{t}, z_{t})$, with transition distribution
\begin{align*}
p(n_{t+1}, \omega_{t+1}, z_{t+1} | n_{t}, \omega_{t}, z_{t}, n_0) = \sum_{m_t} &\pi_\theta\left(a_t = (\scsym{Move}, m_t), z_{t+1} \middle| \omega_t, z_t\right)
\\[-1em]&\qquad \cdot p(n_{t+1} | n_t, m_t) \cdot p(\omega_{t+1} | n_{t+1}, n_t, m_t, n_0)
\end{align*}
and halting distribution
$
\pi_\theta(a_t \in \{\scsym{AddEdgeAndStop},
\scsym{Stop},\scsym{Backtrack}\} ~|~ n_{t}, \omega_{t}, z_{t}).
$
We can represent this distribution via a transition matrix $Q_{n_0} \in \R^{K \times K}$ where $K$ is the set of possible $(n, \omega, z)$ tuples, along with a halting matrix $H \in \R^{(3 \times |N|) \times K}$ (keeping track of the final node $n_T \in N$ as well as the halting action). We can then compute probabilities for each final action by summing over each possible trajectory length $i$:
\begin{align}
p(a_T, n_T | n_0, \pi_\theta)
= \Big[ \sum_{i \ge 0} H Q_{n_0}^i \bvec\delta_{n_0} \Big]_{(a_T, n_T)}
= H_{(a_T, n_T), :} \left(I - Q_{n_0}\right)^{-1} \bvec\delta_{n_0}, \label{eqn:transition-solve-forward}
\end{align}
where $\bvec\delta_{n_0}$ is a vector with a 1 at the position of the initial state tuple $(n_0, \omega_0, z_0)$. Note that, since the matrix depends on the initial state $n_0$, it would be inefficient to analytically invert this matrix for every $n_0$. We thus use $T_{\max}$ iterations (typically 128) of the the Richardson iterative solver \citep{anderssen1972richardson} to obtain an approximate solution using only efficient matrix-vector products; this is equivalent to truncating the sum to include only paths of length at most $T_{\max}$. 

To compute gradients with respect to $\theta$, we use implicit differentiation to express the gradients as the solution to another (transposed) linear system and use the same iterative solver; this ensures that the memory cost of this procedure is independent of $T_{\max}$ (roughly the cost of a single propagation step for a message-passing model).
We implement the forward and backward passes using the automatic differentiation package JAX \citep{jax2018github}, which makes it straightforward to use implicit differentiation with an efficient matrix-vector product implementation that avoids materializing the full transition matrix $Q_{n_0}$ for each value of $n_0$ (see appendix \ref{appendix:gfsa-impl} for details).

\subsection{Absorbing Probabilities as a Derived Adjacency Matrix}\label{subsec:derived-adjacency-matrix}
Finally, we construct an output weighted adjacency matrix by averaging over trajectories:
\begin{equation}\begin{split}
\widehat{A}_{n, n'} &= p(a_T = \scsym{AddEdgeAndStop}, n_T = n' ~|~ n_0=n, a_T \ne \scsym{Backtrack}, z_0, \pi_\theta),
\\
A_{n, n'} &= \sigma\big(a\,\sigma^{-1}\big(\widehat{A}_{n, n'}\big) + b\big)
\end{split}\end{equation}
where $a, b \in \mathbb{R}$ are optional learned adjustment parameters, $\sigma$ denotes the logistic sigmoid, and $\sigma^{-1}$ denotes its inverse. Note that, since they are derived from a probability distribution, the columns of $\widehat{A}_{n, n'}$ sum to at most 1. The adjustment parameters $a$ and $b$ remove this restriction, allowing the model to express one-to-many relationships.

Given a fixed initial automaton state $z_0$, $A_{n, n'}$ can be viewed as a new weighted edge type.
Since $A_{n, n'}$ is differentiable with respect to the policy parameters $\theta$, this adjacency matrix can either be supervised directly or passed to a downstream graph model and trained end-to-end.

\subsection{Connections to Constrained Reachability Problems}
As described in section \ref{subsec:derived-relationships-as-constrained-reachability}, many interesting derived edge types can be expressed as the solutions to constrained reachability problems. Here, we describe a correspondence between constrained reachability problems on graphs and trajectories within the POMDPs defined in section \ref{subsec:from-graphs-to-pomdps}.

\begin{prop}\label{thm:equiv-mdp-language}
Let $\mathcal{G}$ be a family of graphs annotated with node and edge types.
There exists
an encoding of graphs $G \in \mathcal{G}$ into POMDPs as described in section \ref{subsec:from-graphs-to-pomdps}
and a mapping from regular languages $L$ into finite-state policies $\pi_L$
such that,
for any $G \in \mathcal{G}$, there is an $L$-path from $n_0$ to $n_T$ in $G$ if and only if
$p(a_T = \scsym{AddEdgeAndStop}, n_T | n_0, \pi_L) > 0$.
\end{prop}

In other words, for any ordered pair of nodes $(n_0, n_T)$, determining if there is a path in $G$ that satisfies regular-language reachability constraints is equivalent to determining if a specific policy takes the \scsym{AddEdgeAndStop} action at node $n_T$ with nonzero probability when started at node $n_0$, under a particular POMDP representation. 
See appendix \ref{appendix:constrained-reachability-proof} for a proof. As a specific consequence:

\begin{corollary*}
There exists an encoding of program AST graphs into POMDPs and a specific policy $\pi_{\textsc{\tiny NEXT-CF}}$ with finite-state memory such that $p(a_T = \scsym{AddEdgeAndStop}, n_T ~|~ n_0, \pi) > 0$ if and only if $(n_0, n_T)$ is an edge of type \scsym{NextControlFlow} in the augmented AST graph. Similarly, there are policies $\pi_{\textsc{\tiny LAST-READ}}$ and $\pi_{\textsc{\tiny LAST-WRITE}}$ for edges of type \scsym{LastRead} and \scsym{LastWrite}, respectively.
\end{corollary*}

\subsection{Connections to Reinforcement Learning and the Successor Representation} \label{subsec:successor-conection}
The GFSA layer
deterministically computes continuous edge weights by marginalizing over trajectories. These weights can then be transformed nonlinearly
(e.g. $f(\E[\tau])$ where $f$ is the downstream model and loss and $\tau$ are edge additions from trajectories). In contrast, standard RL approaches produce stochastic discrete samples.
As such, is not possible to ``drop in'' an RL approach instead of GFSA;
one must first reformulate the model and task in terms of an expected reward $\E[f(\tau)]$.

Even so, there are interesting connections between the gradient updates for GFSA and traditional RL. In particular,
the columns of the matrix $\left(I-Q_{n_0}\right)^{-1}$ are known in the RL literature as the \emph{successor representation}.
If immediate rewards are described by $\bvec{r}$, then taking a product $\bvec{r}^T \left(I-Q_{n_0}\right)^{-1}$ corresponds to computing the value function \citep{dayan1993improving}. In our case, instead of specifying a reward, we use the GFSA layer for a downstream task that requires optimizing some loss $\mathcal{L}$. When computing gradients of our parameters with respect to $\mathcal{L}$, backpropagation computes a linear approximation of the downstream network and loss function and then uses it in the intermediate expression
\begin{align*}
\frac{\partial{\mathcal{L}}}{\partial p( \cdot | n_0, \pi_\theta)}^T H \left(I-Q_{n_0}\right)^{-1}.
\end{align*}
This is analogous to a non-stationary ``reward function'' for the GFSA policy, which assigns reward to the absorbing states that produce useful edges for the rest of the model. Unlike in standard RL, however, this quantity depends on the full marginal distribution over behaviors. As such, the ``reward'' assigned to a given trajectory may depend on the probability of other, mutually exclusive trajectories.

\section{Related Work}

Some prior work has explored learning edges in graphs. \citet{kipf2018neural} propose a neural relational inference model, which infers pairwise relationships from observed particle trajectories but does not add them to a base graph. \citet{franceschi2019learning} infer missing edges in a single fixed graph by jointly optimizing the edge structure and a classification model; this method only infers edges of a predefined type, and does not generalize to new graphs at test time. \citet{yun2019graph} propose adding new edge types to a graph family by learning to compose a fixed number of existing edge types, which can be seen as a special case of GFSA where each state is visited once. The MINERVA model, described in \citet{das2017go}, uses an RL agent trained with REINFORCE to add edges to a knowledge base, but requires direct supervision of edges. \citet{wang2019learning} use a RL policy to \emph{remove} existing edges from a noisy graph, with reward coming from a downstream classification task.

\citet{bielik2017learning} apply decision trees to program traces with a counterexample-guided program generator in order to learn static analyses of programs. Their method is provably correct, but cannot be used as a component of an end-to-end differentiable model or applied to general graph structures.

Our work shares many commonalities with reinforcement learning techniques. 
Section \ref{subsec:successor-conection} describes a connection between the GFSA computation and the successor representation \citep{dayan1993improving}.
Our work is also conceptually similar to methods for learning options. For instance, \citet{bacon2017option} describe an end-to-end architecture for learning options by differentiating through a primary policy's reward. Their option policies and primary policy are analogous to our GFSA edge types and downstream model; on the other hand, they apply policy gradient methods to trajectory samples instead of optimizing over full marginal distributions, and their full architecture is still a policy, not a general model on graphs.

Existing graph embedding methods have used stochastic walks on graphs \citep{ying2018graph,ivanov2018anonymous,zhang2018retgk,jiang2019data, huang2019graph,busch2020pushnet}, but generally assume uniform random walks. \citet{alon2018code2seq} propose representing ASTs by sampling random paths and concatenating their node labels, then attending over the resulting  sequences. \citet{dai2019learning} describe a framework of MDPs over graphs, but focus on a ``learning to explore'' task, where the goal is to visit many nodes and the agent can see the entire subgraph it has already visited. \citet{hudson2019learning} propose treating the nodes of an inferred scene graph as states of a learned state machine, and learning to update the current active node based on natural-language inputs.

Self-attention can be viewed as constructing a weighted adjacency matrix similar to GFSA, but only considers pairwise relationships and not longer paths. Existing approaches to learning multi-step path-based relationships include
iterating a graph neural network until convergence \citep{scarselli2008graph} and using a learned stopping criterion as in Universal Transformers \citep{dehghani2018universal}.
The algorithm in section \ref{subsec:computing-absorbing} in particular resembles
running a separate graph neural network model to convergence for each start node and training with recurrent backpropagation \citep{scarselli2008graph, liao2018reviving}, and is also similar to other uses of implicit differentiation \citep{wilder2019end,bai2019deep,rajeswaran2019meta}. The GFSA layer enables multi-step relationships to be efficiently computed for every start node in parallel and provides good expressivity and inductive biases for learning edges, in contrast to previous techniques that focus on learning node representations and must learn from scratch to propagate multi-step information without letting distinct paths interfere with each other.

\citet{weiss2018extracting} describe a method for extracting a discrete finite-state automaton from a RNN; this assumes access to an existing trained RNN for the task, and is intended for recognizing sequences, not adding edges to graphs.
See also
\citet{mohri2009weighted} for a framework of weighted automata on sequences.

\section{Experiments}

\subsection{Grid-World Options}
As an illustrative example, we consider the task of discovering useful navigation strategies in grid-world environments. We generate grid-world layouts using the LabMaze generator \citep{beattie2016deepmind},\footnote{\url{https://github.com/deepmind/labmaze}} and interpret each cell as a node in a graph $G$, where edges represent cardinal directions. We augment this graph with additional edges from a GFSA layer, using four independent GFSA policies to add four additional edge types;
let $G'_\theta(G)$ denote the augmented graph using GFSA parameters $\theta$.
Next, we construct a pathfinding task on the augmented graph $G'_\theta(G)$, in which a graph-specific agent finds the shortest path to some goal node $g$. We assign an equal cost to all edges (including those that the GFSA layer adds); when the agent follows a GFSA edge, it ends up at a destination cell with probability proportional to the edge weights from the GFSA layer.

Inspired by existing work on meta-learning options \citep{frans2017meta}, we interpret the GFSA-derived edges as a kind of option for this agent: given a random graph, the edges added by the GFSA layer should make it possible to quickly reach any goal node $g$ from any start location $n_0$. More specifically, we train the graph-independent GFSA layer (in an outer loop) to minimize the number of steps that a graph-specific policy (trained in an inner loop) takes to reach the goal $g$, i.e. we minimize
\begin{align*}
\mathcal{L} = \mathbb{E}_{G,n_0,g}\left[\mathbb{E}_{n_t \sim \pi^*(\cdot | n_{t-1}, g, G'_\theta(G))}\left[T ~\middle|~ n_T = g\right]\right]
\end{align*}
where $\pi^*(\cdot | g, G'_\theta(G))$ is an optimal tabular policy for graph $G'_\theta(G)$ and goal $g$. In order to differentiate this with respect to the GFSA parameters $\theta$, we use entropy regularization to ensure $\pi^*(\cdot | g, G'_\theta(G))$ is smooth, and solve for it by iterating the soft Bellman equation until convergence \citep{haarnoja2017reinforcement}, again using implicit differentiation to backpropagate through that solution (see appendix \ref{appendix:gridworld}).

Figure \ref{fig:gridworld} shows the derived edges learned by the GFSA layer
on a graph not seen during training; we find that the edges learned by the GFSA layer are discrete and roughly correspond to diagonal motions in the grid. Over the course of training the GFSA layer, the average number of steps taken by the (optimal) primary policy (on a validation set of unseen layouts) decreases from 40.1 steps to 11.5 steps, a substantial improvement in the end-to-end performance. This example illustrates the kind of relationships the GFSA layer can learn from end-to-end supervision; note that we do not claim these options are optimal for this task or would be practical in a more traditional RL context. 

\subsection{Learning Static Analyses of Python Code}\label{subsec:static-analyses}
Proposition \ref{thm:equiv-mdp-language} ensures that a GFSA is theoretically capable of performing simple static analyses of code. We demonstrate that the GFSA can practically learn to do these analyses by casting them as pairwise binary classification problems. We first generate a synthetic dataset of Python programs by sampling from a probabilistic context-free grammar over a subset of Python.
We then transform the corresponding ASTs into graphs, and compute the three edge types \scsym{NextControlFlow}, \scsym{LastRead}, and \scsym{LastWrite}, which are commonly used for program understanding tasks \citep{allamanis2018learning,hellendoorn2020global} and which we describe in section \ref{subsec:derived-relationships-as-constrained-reachability}.
Note that there may be multiple edges from the same statement or variable, since there are often multiple possible execution paths through the program.

For each of these edge types, we train a GFSA layer to classify whether each ordered pair of nodes is connected with an edge of that type. We use the focal-loss objective \citep{lin2017focal}, a more stable variant of the cross-entropy loss for highly unbalanced classification problems, minimizing
\[
\mathcal{L} = \mathbb{E}_{(N,E) \sim \mathcal{D}}\left[\,\sum_{n_1, n_2 \in N}
\begin{cases}
-(1-A_{n_1,n_2})^\gamma \log(A_{n_1,n_2}) &\text{if }(n_1 \to n_2) \in E,\\
-(A_{n_1,n_2})^\gamma \log(1 - A_{n_1,n_2}) &\text{otherwise}
\end{cases}
\right]
\]
where the expectation is taken over graphs in the training dataset $\mathcal{D}$.

We compare against four graph model baselines: a GGNN \citep{li2015gated}, a GREAT model over AST graphs \citep{hellendoorn2020global}, a RAT model \citep{wang2019rat}, and an NRI-style encoder  \citep{kipf2018neural}. For the GGNN, GREAT, and RAT models, we present results for two methods of computing output adjacency matrices: the first computes a learned key-value dot product (similar to dot-product attention) and interprets it as an adjacency matrix, and the second runs the model separately for each possible source node, tagging that source with an extra node feature, and computing an output for each possible destination (denoted ``nodewise''). For the NRI encoder model, the output head is an MLP over node feature pairs as described by \citet{kipf2018neural}; we extend the NRI model with residual connections and layer normalization to improve stability, similar to a transformer model \citep{vaswani2017attention}. All baselines use a logistic sigmoid as a final activation, and are trained with the focal-loss objective. See appendix \ref{appendix:edge-classification} for more details.

As an ablation, we also train a standard RL agent with the same parameterization as GFSA, inspired by MINERVA \citep{das2017go}. We replace the cross-entropy loss with a reward of +1 for adding a correct edge (or correctly not adding any) and 0 otherwise, and train using REINFORCE with 20 rollouts per start node and a leave-one-out control variate \citep{williams1992simple,kool2019buy}. Since edges are added by single trajectories rather than marginals over trajectories, this RL agent can add at most one edge from each start node.

\begin{table}
\centering
\caption{Results on the program analysis edge-classification tasks. Values are F1 scores (in percent), with bold indicating overlapping 95\% confidence intervals with the best model; see appendix \ref{appendix:edge-classification-detailed-results} for full-precision results. ``nw'' denotes nodewise output, and ``dp'' denotes dot-product output.}
\begin{tabular}{rrrrrrrrrr}
\toprule
\multicolumn{10}{c}{\textbf{100,000 training examples}}\\
\midrule
\textbf{Task} & \multicolumn{3}{c}{Next Control Flow} &\multicolumn{3}{c}{Last Read} & \multicolumn{3}{c}{Last Write} \\
\cmidrule(r{4pt}){2-4} \cmidrule(l{4pt} r{4pt}){5-7} \cmidrule(l{4pt}){8-10}
\textbf{Example size} &
\multicolumn{1}{c}{1x}&
\multicolumn{1}{c}{2x}&
\multicolumn{1}{c}{0.5x}&
\multicolumn{1}{c}{1x}&
\multicolumn{1}{c}{2x}&
\multicolumn{1}{c}{0.5x}&
\multicolumn{1}{c}{1x}&
\multicolumn{1}{c}{2x}&
\multicolumn{1}{c}{0.5x}\\
\midrule
\emph{RAT nw} &
99.98 & 99.94 & 99.99 & 99.86 & 96.29 & \textbf{99.98} & 99.83 & 94.87 & \textbf{99.97} \\
\emph{GREAT nw} &
99.98 & 99.87 & 99.98 & 99.91 & 95.12 & \textbf{99.98} & 99.75 & 93.22 & 99.93 \\
\emph{GGNN nw} &
99.98 & 93.90 & 97.77 & 95.52 & 9.22 & 86.24 & 98.82 & 40.69 & 88.28 \\
\emph{RAT dp} &
99.99 & 92.53 & 96.59 & 99.96 & 42.58 & 91.96 & 99.98 & 68.96 & 99.76 \\
\emph{GREAT dp} &
99.99 & 96.32 & 98.36 & \textbf{99.99} & 47.07 & 99.78 & \textbf{99.99} & 68.46 & 99.88 \\
\emph{GGNN dp} &
99.94 & 62.75 & 98.51 & 98.44 & 0.99 & 63.77 & 99.35 & 38.40 & 94.52 \\
\emph{NRI encoder} &
99.98 & 85.91 & 99.92 & 99.83 & 43.44 & 99.39 & 99.87 & 52.73 & 99.84 \\
\emph{RL ablation} &
94.24 & 93.56 & 94.83 & 96.69 & 94.85 & 97.85 & 98.08 & 96.64 & 98.93 \\
\emph{GFSA (ours)} &
\textbf{100.00} & \textbf{99.99} & \textbf{100.00} & 99.66 & \textbf{98.94} & 99.90 & 99.47 & \textbf{98.73} & 99.78 \\
\midrule
\multicolumn{10}{c}{\textbf{100 training examples}}\\
\midrule
\emph{RAT nw} &
98.63 & 95.93 & 96.32 & 80.28 & 1.12 & 83.49 & 79.27 & 8.91 & 83.79 \\
\emph{GREAT nw} &
98.23 & 97.98 & 98.52 & 78.88 & 6.96 & 60.90 & 80.19 & 40.22 & 84.54 \\
\emph{GGNN nw} &
99.37 & 98.36 & 98.60 & 79.36 & 28.28 & 5.66 & 91.13 & 71.62 & 91.79 \\
\emph{RAT dp} &
81.81 & 68.46 & 87.05 & 59.53 & 28.91 & 62.27 & 75.99 & 48.10 & 81.63 \\
\emph{GREAT dp} &
86.60 & 62.98 & 80.58 & 57.02 & 27.13 & 64.48 & 73.69 & 46.27 & 80.03 \\
\emph{GGNN dp} &
76.85 & 22.99 & 28.91 & 44.37 & 9.64 & 38.34 & 53.82 & 17.84 & 55.08 \\
\emph{NRI encoder} &
81.74 & 69.08 & 88.87 & 68.69 & 26.64 & 73.52 & 65.38 & 36.43 & 73.86 \\
\emph{RL ablation} &
91.70 & 91.14 & 92.29 & 98.48 & 97.03 & 99.17 & 98.32 & \textbf{96.96} & 99.07 \\
\emph{GFSA (ours)} &
\textbf{99.99} & \textbf{99.99} & \textbf{100.00} & \textbf{98.81} & \textbf{97.82} & \textbf{99.22} & \textbf{98.71} & \textbf{96.98} & \textbf{99.55} \\
\bottomrule
\end{tabular}
    \label{tab:edge_supervision}
\end{table}

Table \ref{tab:edge_supervision} shows results of each of these models on the three edge classification tasks. We present results after training on a dataset of 100,000 examples as well as on a smaller dataset of only 100 examples, and report F1 scores at the best classification threshold; we choose the model with the best validation performance from a 32-job random hyperparameter search. To assess generalization, we also show results on two modified data distributions: programs of half the size of those in the training set (0.5x), and programs twice the size (2x).
When trained on 100,000 examples, all models achieve high accuracy on examples of the training size, but some fail to generalize, especially to larger programs. When trained on 100 examples, only the GFSA
layer and RL ablation consistently achieve high accuracy, highlighting the strong inductive bias for constrained-reachability-based reasoning tasks.
The GFSA layer trained with exact marginals and cross-entropy loss obtains higher accuracy than the RL ablation, and also converges more reliably:
82\% of GFSA layer training jobs achieve at least 90\% accuracy on the validation set, compared to only 11\% of RL ablation jobs.

Figure \ref{fig:lastread} shows an example of the behavior that the GFSA layer learns for the \scsym{LastRead} task based on only input-output supervision. We note that the GFSA layer discovers separate modes for break statements and regular control flow, and also learns to split probability mass across multiple trajectories in order to account for multiple paths through the program, closely following the program semantics. The paths learned by this policy are also quite long; the policy shown takes an average of 35 actions before accepting (on the 1x test set). More generally, this shows that the GFSA layer is able to learn many-hop reasoning that covers large distances in the graph by breaking down the reasoning into subcomponents defined by the learned automaton states.

\subsection{Variable Misuse}\label{subsec:var_misuse}

\begin{table}[t]
\centering
\caption{
Accuracy on the variable misuse task, in percent. ``Start'' indicates that edges are added to the base graph before running the graph model, and ``middle'' indicates they are added halfway through, conditioned on the output of the first half. Bold indicates overlapping 95\% confidence intervals with the best model for each metric. See appendix \ref{appendix:var-misuse-detailed-results} for  standard error estimates and additional details.
}
\label{tab:var_misuse}

\setlength{\tabcolsep}{3pt}
\begin{tabular}{l cc c cc c cc c cc}
\toprule
Example type:
& \multicolumn{2}{c}{All}
&& \multicolumn{2}{c}{No bug}
&& \multicolumn{5}{c}{With bug}\\
\cmidrule{2-3}\cmidrule{5-6}\cmidrule{8-12}
Metric: & \multicolumn{2}{c}{Full accuracy}
&& \multicolumn{2}{c}{Classification}
&& \multicolumn{2}{c}{Classification}
&& \multicolumn{2}{c}{Loc \& Repair} \\
\cmidrule{2-3}\cmidrule{5-6}\cmidrule{8-9}\cmidrule{11-12}
Graph model family: & \emph{RAT} & \emph{GGNN} && \emph{RAT} & \emph{GGNN} && \emph{RAT} & \emph{GGNN} && \emph{RAT} & \emph{GGNN} \\
\midrule

\emph{Base AST graph only}
 & {88.22} & {83.52} &  & {92.05} & {91.26} &  & {93.03} & {88.15} &  & {88.30} & {81.63} \\
\emph{Base AST graph, +2 layers}
 & {87.85} & {84.38} &  & {92.45} & {88.80} &  & {92.03} & {91.92} &  & {87.76} & {83.97} \\
\emph{Hand-engineered edges}
 & {88.50} & {84.78} &  & {92.93} & {90.19} &  & {92.48} & {91.56} &  & {88.39} & {83.52} \\
\emph{NRI head @ start}
 & {88.71} & {84.47} &  & {92.55} & {91.49} &  & {93.21} & {89.38} &  & {88.73} & {82.73} \\
\emph{NRI head @ middle}
 & {88.42} & {84.41} &  & {92.83} & {88.29} &  & {92.31} & {92.20} &  & {88.62} & {84.44} \\
\emph{Random walk @ start}
 & {88.91} & {84.52} &  & \textbf{93.22} & {91.35} &  & {92.77} & {89.28} &  & {88.73} & {82.96} \\
\emph{RL ablation @ middle}
 & {87.28} & {84.96} &  & {90.36} & {90.44} &  & {93.71} & {90.64} &  & {87.73} & {84.30} \\
\emph{GFSA layer (ours) @ start}
 & \textbf{89.47} & {85.01} &  & \textbf{93.10} & {90.08} &  & {93.56} & {91.80} &  & {89.58} & {83.91} \\
\emph{GFSA layer (ours) @ middle}
 & \textbf{89.63} & {84.72} &  & {92.66} & {90.98} &  & \textbf{94.25} & {89.81} &  & \textbf{89.93} & {83.63} \\

\bottomrule
\end{tabular}

\end{table}

Finally, we investigate performance on the variable misuse task \citep{allamanis2018learning,vasic2019neural}. 
Following \citet{hellendoorn2020global}, we 
use a dataset of small code samples from a permissively-licenced subset of the ETH 150k Python dataset \citep{raychev2016probabilistic}, where synthetic variable misuse bugs have been introduced in half of the examples by randomly replacing one of the identifiers with a different identifier in that program.%
\footnote{\url{https://github.com/google-research-datasets/great}}
We train a model to predict the location of the incorrect identifier, as well as another location in the program containing the correct replacement that would restore the original program; we use a special ``no-bug'' location for the unmodified examples, similar to \citet{vasic2019neural} and \citet{hellendoorn2020global}.

We consider two graph neural network architectures: either an eight-layer RAT model \citep{wang2019rat} or eight GGNN blocks \citep{li2015gated} with two message passing iterations per block (similar to \citet{hellendoorn2020global}).
For each, we investigate adding different types of edges to the base AST graph:
no extra edges, hand-engineered edges used by \citet{allamanis2018learning} and \citet{hellendoorn2020global}, weighted edges learned by a GFSA layer, weighted edges output by an NRI-like pairwise MLP, weighted edges produced by an ablation of GFSA consisting of a uniform random walk with a learned halting probability, and a single edge per start state sampled by a GFSA-based RL agent. For the NRI and GFSA layers, we investigate adding the edges either 
before the graph neural network model (building from the base graph), or halfway through the model (conditioned on the node embeddings from the first half).
For the RL agent, we train with REINFORCE and a learned scalar reward baseline, and use the downstream cross-entropy loss as the reward.
To show the effect of just increasing model capacity, we also present results for ten-layer models on the base graph.
In all models, we initialize node embeddings based on a subword tokenization of the program (using the \texttt{Tensor2Tensor} library by \citet{vaswani2018tensor2tensor}), and predict a joint distribution over the bug and repair locations, with softmax normalization and the standard cross entropy objective. See appendix \ref{appendix:var-misuse} for additional details on each of the above models, as well as results using an eight-layer GREAT model \citep{hellendoorn2020global}.

The results are shown in Table \ref{tab:var_misuse}.
We report overall accuracy, along with a breakdown by example type: for non-buggy examples, we report the fraction of examples the model predicts as non-buggy, and for buggy examples, we report both accuracy of the classification and accuracy of the predicted error and replacement identifier locations conditioned on the classification.
Consistent with prior work, adding the hand-engineered features from \citet{allamanis2018learning} improves performance over only using the base graph.
Interestingly, adding weighted edges using a random walk on the base graph yields similar performance to adding hand-engineered edges, suggesting that, for this task, improving connectivity may be more important than the specific program analyses used.
We find that the GFSA layer combined with the RAT graph model obtains the best performance, outperforming the hand-engineered edges.
Interestingly, we observe that the GFSA layer does not seem to converge to a discrete adjacency matrix, but instead assigns continuous weights. We conjecture that the output edge weights may provide additional representative power to the base model.

\section{Conclusion}
Inspired by ideas from programming languages and reinforcement learning, we propose the differentiable GFSA layer, which learns to add new edges to a base graph.
We show that the GFSA layer can learn sophisticated behaviors for navigating grid-world environments and analyzing program behavior, and demonstrate that it can act as a viable replacement for hand-engineered edges in the variable misuse task.
In the future, we plan to apply the GFSA layer to other domains and tasks, such as molecular structures or larger code repositories. We also hope to investigate the interpretability of the edges learned by the GFSA layer to determine whether they correspond to useful general concepts, which might allow the GFSA edges to be shared between multiple tasks.

\section*{Broader Impact}
We consider this work to be a general technical and theoretical contribution, without well-defined specific impacts.
If applied to real-world program understanding tasks, extensions of this work might lead to reduced bug frequency or improved developer productivity. On the other hand, those benefits might accrue mostly to groups with sufficient resources to incorporate machine learning into their development practices. Additionally, if users put too much trust in the output of the model, they could inadvertently introduce bugs in their code because of incorrect model predictions. If applied to other tasks involving structured data, the impact would depend on the specific application; we leave the exploration of these other applications and their potential impacts to future work.

\section*{Acknowledgments}
We would like to thank Aditya Kanade and Charles Sutton for pointing out the connection to \citet{reps1998program}, and Petros Maniatis for help with the variable misuse dataset. We would also like to thank Dibya Ghosh and Yujia Li for their helpful comments and suggestions during the writing process, and the Brain Program Learning, Understanding, and Synthesis team at Google for useful feedback throughout the course of the project.
Finally, we thank the reviewers for their feedback and for pointing out relevant related work.

\small
\bibliographystyle{plainnat}
\bibliography{references}

\begin{thebibliography}{49}
\providecommand{\natexlab}[1]{#1}
\providecommand{\url}[1]{\texttt{#1}}
\expandafter\ifx\csname urlstyle\endcsname\relax
  \providecommand{\doi}[1]{doi: #1}\else
  \providecommand{\doi}{doi: \begingroup \urlstyle{rm}\Url}\fi

\bibitem[Allamanis et~al.(2018)Allamanis, Brockschmidt, and
  Khademi]{allamanis2018learning}
Miltiadis Allamanis, Marc Brockschmidt, and Mahmoud Khademi.
\newblock Learning to represent programs with graphs.
\newblock In \emph{International Conference on Learning Representations}, 2018.

\bibitem[Alon et~al.(2019)Alon, Brody, Levy, and Yahav]{alon2018code2seq}
Uri Alon, Shaked Brody, Omer Levy, and Eran Yahav.
\newblock code2seq: Generating sequences from structured representations of
  code.
\newblock In \emph{International Conference on Learning Representations}, 2019.

\bibitem[Anderssen and Golub(1972)]{anderssen1972richardson}
RS~Anderssen and GH~Golub.
\newblock Richardson's non-stationary matrix iterative procedure. rep.
\newblock Technical report, STAN-CS-72-304, Computer Science Dept., Stanford
  University Report, 1972.

\bibitem[Bacon et~al.(2017)Bacon, Harb, and Precup]{bacon2017option}
Pierre-Luc Bacon, Jean Harb, and Doina Precup.
\newblock The option-critic architecture.
\newblock In \emph{Thirty-First AAAI Conference on Artificial Intelligence},
  2017.

\bibitem[Bai et~al.(2019)Bai, Kolter, and Koltun]{bai2019deep}
Shaojie Bai, J~Zico Kolter, and Vladlen Koltun.
\newblock Deep equilibrium models.
\newblock In \emph{Advances in Neural Information Processing Systems}, pages
  688--699, 2019.

\bibitem[Beattie et~al.(2016)Beattie, Leibo, Teplyashin, Ward, Wainwright,
  K{\"u}ttler, Lefrancq, Green, Vald{\'e}s, Sadik, et~al.]{beattie2016deepmind}
Charles Beattie, Joel~Z Leibo, Denis Teplyashin, Tom Ward, Marcus Wainwright,
  Heinrich K{\"u}ttler, Andrew Lefrancq, Simon Green, V{\'\i}ctor Vald{\'e}s,
  Amir Sadik, et~al.
\newblock Deepmind lab.
\newblock \emph{arXiv preprint arXiv:1612.03801}, 2016.

\bibitem[Bielik et~al.(2017)Bielik, Raychev, and Vechev]{bielik2017learning}
Pavol Bielik, Veselin Raychev, and Martin Vechev.
\newblock Learning a static analyzer from data.
\newblock In \emph{International Conference on Computer Aided Verification},
  pages 233--253. Springer, 2017.

\bibitem[Bradbury et~al.(2018)Bradbury, Frostig, Hawkins, Johnson, Leary,
  Maclaurin, and Wanderman-Milne]{jax2018github}
James Bradbury, Roy Frostig, Peter Hawkins, Matthew~James Johnson, Chris Leary,
  Dougal Maclaurin, and Skye Wanderman-Milne.
\newblock {JAX}: composable transformations of {P}ython+{N}um{P}y programs.
\newblock \url{http://github.com/google/jax}, 2018.

\bibitem[Busch et~al.(2020)Busch, Pi, and Seidl]{busch2020pushnet}
Julian Busch, Jiaxing Pi, and Thomas Seidl.
\newblock Pushnet: Efficient and adaptive neural message passing.
\newblock \emph{arXiv preprint arXiv:2003.02228}, 2020.

\bibitem[Cousot and Cousot(1977)]{cousot1977abstract}
Patrick Cousot and Radhia Cousot.
\newblock Abstract interpretation: a unified lattice model for static analysis
  of programs by construction or approximation of fixpoints.
\newblock In \emph{Proceedings of the 4th ACM SIGACT-SIGPLAN symposium on
  Principles of programming languages}, pages 238--252, 1977.

\bibitem[Dai et~al.(2019)Dai, Li, Wang, Singh, Huang, and
  Kohli]{dai2019learning}
Hanjun Dai, Yujia Li, Chenglong Wang, Rishabh Singh, Po-Sen Huang, and Pushmeet
  Kohli.
\newblock Learning transferable graph exploration.
\newblock In \emph{Advances in Neural Information Processing Systems}, pages
  2514--2525, 2019.

\bibitem[Das et~al.(2018)Das, Dhuliawala, Zaheer, Vilnis, Durugkar,
  Krishnamurthy, Smola, and McCallum]{das2017go}
Rajarshi Das, Shehzaad Dhuliawala, Manzil Zaheer, Luke Vilnis, Ishan Durugkar,
  Akshay Krishnamurthy, Alex Smola, and Andrew McCallum.
\newblock Go for a walk and arrive at the answer: Reasoning over paths in
  knowledge bases using reinforcement learning.
\newblock In \emph{International Conference on Learning Representations}, 2018.

\bibitem[Dayan(1993)]{dayan1993improving}
Peter Dayan.
\newblock Improving generalization for temporal difference learning: The
  successor representation.
\newblock \emph{Neural Computation}, 5\penalty0 (4):\penalty0 613--624, 1993.

\bibitem[Dehghani et~al.(2018)Dehghani, Gouws, Vinyals, Uszkoreit, and
  Kaiser]{dehghani2018universal}
Mostafa Dehghani, Stephan Gouws, Oriol Vinyals, Jakob Uszkoreit, and {\L}ukasz
  Kaiser.
\newblock Universal transformers.
\newblock \emph{arXiv preprint arXiv:1807.03819}, 2018.

\bibitem[Franceschi et~al.(2019)Franceschi, Niepert, Pontil, and
  He]{franceschi2019learning}
Luca Franceschi, Mathias Niepert, Massimiliano Pontil, and Xiao He.
\newblock Learning discrete structures for graph neural networks.
\newblock \emph{arXiv preprint arXiv:1903.11960}, 2019.

\bibitem[Frans et~al.(2018)Frans, Ho, Chen, Abbeel, and
  Schulman]{frans2017meta}
Kevin Frans, Jonathan Ho, Xi~Chen, Pieter Abbeel, and John Schulman.
\newblock Meta learning shared hierarchies.
\newblock In \emph{International Conference on Learning Representations}, 2018.

\bibitem[Gilmer et~al.(2017)Gilmer, Schoenholz, Riley, Vinyals, and
  Dahl]{gilmer2017neural}
Justin Gilmer, Samuel~S Schoenholz, Patrick~F Riley, Oriol Vinyals, and
  George~E Dahl.
\newblock Neural message passing for quantum chemistry.
\newblock In \emph{Proceedings of the 34th International Conference on Machine
  Learning-Volume 70}, pages 1263--1272. JMLR. org, 2017.

\bibitem[Haarnoja et~al.(2017)Haarnoja, Tang, Abbeel, and
  Levine]{haarnoja2017reinforcement}
Tuomas Haarnoja, Haoran Tang, Pieter Abbeel, and Sergey Levine.
\newblock Reinforcement learning with deep energy-based policies.
\newblock In \emph{Proceedings of the 34th International Conference on Machine
  Learning-Volume 70}, pages 1352--1361. JMLR. org, 2017.

\bibitem[Hellendoorn et~al.(2020)Hellendoorn, Sutton, Singh, and
  Maniatis]{hellendoorn2020global}
Vincent~J Hellendoorn, Charles Sutton, Rishabh Singh, and Petros Maniatis.
\newblock Global relational models of source code.
\newblock In \emph{International Conference on Learning Representations}, 2020.

\bibitem[Hopcroft et~al.(2001)Hopcroft, Motwani, and
  Ullman]{hopcroft2001introduction}
John~E Hopcroft, Rajeev Motwani, and Jeffrey~D Ullman.
\newblock Introduction to automata theory, languages, and computation.
\newblock \emph{Acm Sigact News}, 32\penalty0 (1):\penalty0 60--65, 2001.

\bibitem[Huang et~al.(2019)Huang, Song, Li, and Hu]{huang2019graph}
Xiao Huang, Qingquan Song, Yuening Li, and Xia Hu.
\newblock Graph recurrent networks with attributed random walks.
\newblock In \emph{Proceedings of the 25th ACM SIGKDD International Conference
  on Knowledge Discovery \& Data Mining}, pages 732--740, 2019.

\bibitem[Hudson and Manning(2019)]{hudson2019learning}
Drew Hudson and Christopher~D Manning.
\newblock Learning by abstraction: The neural state machine.
\newblock In \emph{Advances in Neural Information Processing Systems}, pages
  5903--5916, 2019.

\bibitem[Ivanov and Burnaev(2018)]{ivanov2018anonymous}
Sergey Ivanov and Evgeny Burnaev.
\newblock Anonymous walk embeddings.
\newblock In \emph{International Conference on Machine Learning}, pages
  2186--2195, 2018.

\bibitem[Jiang et~al.(2019)Jiang, Lin, Tang, and Luo]{jiang2019data}
Bo~Jiang, Doudou Lin, Jin Tang, and Bin Luo.
\newblock Data representation and learning with graph diffusion-embedding
  networks.
\newblock In \emph{Proceedings of the IEEE Conference on Computer Vision and
  Pattern Recognition}, pages 10414--10423, 2019.

\bibitem[Kipf et~al.(2018)Kipf, Fetaya, Wang, Welling, Zemel,
  et~al.]{kipf2018neural}
T~Kipf, E~Fetaya, K-C Wang, M~Welling, R~Zemel, et~al.
\newblock Neural relational inference for interacting systems.
\newblock \emph{Proceedings of Machine Learning Research}, 80, 2018.

\bibitem[Kool et~al.(2019)Kool, van Hoof, and Welling]{kool2019buy}
Wouter Kool, Herke van Hoof, and Max Welling.
\newblock Buy 4 {REINFORCE} samples, get a baseline for free!
\newblock \emph{{ICLR} workshop: Deep {RL} Meets Structured Prediction}, 2019.

\bibitem[Li et~al.(2015)Li, Tarlow, Brockschmidt, and Zemel]{li2015gated}
Yujia Li, Daniel Tarlow, Marc Brockschmidt, and Richard Zemel.
\newblock Gated graph sequence neural networks.
\newblock \emph{arXiv preprint arXiv:1511.05493}, 2015.

\bibitem[Liao et~al.(2018)Liao, Xiong, Fetaya, Zhang, Yoon, Pitkow, Urtasun,
  and Zemel]{liao2018reviving}
Renjie Liao, Yuwen Xiong, Ethan Fetaya, Lisa Zhang, KiJung Yoon, Xaq Pitkow,
  Raquel Urtasun, and Richard Zemel.
\newblock Reviving and improving recurrent back-propagation.
\newblock In \emph{International Conference on Machine Learning}, pages
  3082--3091, 2018.

\bibitem[Lin et~al.(2017)Lin, Goyal, Girshick, He, and
  Doll{\'a}r]{lin2017focal}
Tsung-Yi Lin, Priya Goyal, Ross Girshick, Kaiming He, and Piotr Doll{\'a}r.
\newblock Focal loss for dense object detection.
\newblock In \emph{Proceedings of the IEEE international conference on computer
  vision}, pages 2980--2988, 2017.

\bibitem[Mnih et~al.(2013)Mnih, Kavukcuoglu, Silver, Graves, Antonoglou,
  Wierstra, and Riedmiller]{mnih2013playing}
Volodymyr Mnih, Koray Kavukcuoglu, David Silver, Alex Graves, Ioannis
  Antonoglou, Daan Wierstra, and Martin Riedmiller.
\newblock Playing atari with deep reinforcement learning.
\newblock \emph{arXiv preprint arXiv:1312.5602}, 2013.

\bibitem[Mohri(2009)]{mohri2009weighted}
Mehryar Mohri.
\newblock Weighted automata algorithms.
\newblock In \emph{Handbook of weighted automata}, pages 213--254. Springer,
  2009.

\bibitem[Rajeswaran et~al.(2019)Rajeswaran, Finn, Kakade, and
  Levine]{rajeswaran2019meta}
Aravind Rajeswaran, Chelsea Finn, Sham~M Kakade, and Sergey Levine.
\newblock Meta-learning with implicit gradients.
\newblock In \emph{Advances in Neural Information Processing Systems}, pages
  113--124, 2019.

\bibitem[Raychev et~al.(2016)Raychev, Bielik, and
  Vechev]{raychev2016probabilistic}
Veselin Raychev, Pavol Bielik, and Martin Vechev.
\newblock Probabilistic model for code with decision trees.
\newblock \emph{ACM SIGPLAN Notices}, 51\penalty0 (10):\penalty0 731--747,
  2016.

\bibitem[Reps(1998)]{reps1998program}
Thomas Reps.
\newblock Program analysis via graph reachability.
\newblock \emph{Information and software technology}, 40\penalty0
  (11-12):\penalty0 701--726, 1998.

\bibitem[Saad(2003)]{saad2003iterative}
Yousef Saad.
\newblock \emph{Iterative methods for sparse linear systems}, volume~82.
\newblock SIAM, 2003.

\bibitem[Scarselli et~al.(2008)Scarselli, Gori, Tsoi, Hagenbuchner, and
  Monfardini]{scarselli2008graph}
Franco Scarselli, Marco Gori, Ah~Chung Tsoi, Markus Hagenbuchner, and Gabriele
  Monfardini.
\newblock The graph neural network model.
\newblock \emph{IEEE Transactions on Neural Networks}, 20\penalty0
  (1):\penalty0 61--80, 2008.

\bibitem[Shaw et~al.(2018)Shaw, Uszkoreit, and Vaswani]{shaw2018self}
Peter Shaw, Jakob Uszkoreit, and Ashish Vaswani.
\newblock Self-attention with relative position representations.
\newblock In \emph{Proceedings of the 2018 Conference of the North American
  Chapter of the Association for Computational Linguistics: Human Language
  Technologies, Volume 2 (Short Papers)}, pages 464--468, 2018.

\bibitem[Sutton et~al.(1999)Sutton, Precup, and Singh]{sutton1999between}
Richard~S Sutton, Doina Precup, and Satinder Singh.
\newblock Between mdps and semi-mdps: A framework for temporal abstraction in
  reinforcement learning.
\newblock \emph{Artificial intelligence}, 112\penalty0 (1-2):\penalty0
  181--211, 1999.

\bibitem[Vasic et~al.(2019)Vasic, Kanade, Maniatis, Bieber, and
  Singh]{vasic2019neural}
Marko Vasic, Aditya Kanade, Petros Maniatis, David Bieber, and Rishabh Singh.
\newblock Neural program repair by jointly learning to localize and repair.
\newblock In \emph{International Conference on Learning Representations}, 2019.

\bibitem[Vaswani et~al.(2017)Vaswani, Shazeer, Parmar, Uszkoreit, Jones, Gomez,
  Kaiser, and Polosukhin]{vaswani2017attention}
Ashish Vaswani, Noam Shazeer, Niki Parmar, Jakob Uszkoreit, Llion Jones,
  Aidan~N Gomez, {\L}ukasz Kaiser, and Illia Polosukhin.
\newblock Attention is all you need.
\newblock In \emph{Advances in Neural Information Processing Systems}, pages
  5998--6008, 2017.

\bibitem[Vaswani et~al.(2018)Vaswani, Bengio, Brevdo, Chollet, Gomez, Gouws,
  Jones, Kaiser, Kalchbrenner, Parmar, et~al.]{vaswani2018tensor2tensor}
Ashish Vaswani, Samy Bengio, Eugene Brevdo, Francois Chollet, Aidan Gomez,
  Stephan Gouws, Llion Jones, {\L}ukasz Kaiser, Nal Kalchbrenner, Niki Parmar,
  et~al.
\newblock Tensor2tensor for neural machine translation.
\newblock In \emph{Proceedings of the 13th Conference of the Association for
  Machine Translation in the Americas (Volume 1: Research Papers)}, pages
  193--199, 2018.

\bibitem[Wang et~al.(2019{\natexlab{a}})Wang, Shin, Liu, Polozov, and
  Richardson]{wang2019rat}
Bailin Wang, Richard Shin, Xiaodong Liu, Oleksandr Polozov, and Matthew
  Richardson.
\newblock {RAT-SQL}: Relation-aware schema encoding and linking for
  text-to-{SQL} parsers.
\newblock \emph{arXiv preprint arXiv:1911.04942}, 2019{\natexlab{a}}.

\bibitem[Wang et~al.(2019{\natexlab{b}})Wang, Yu, Wang, Cheng, Zhang, Zha, He,
  and Chen]{wang2019learning}
Lu~Wang, Wenchao Yu, Wei Wang, Wei Cheng, Wei Zhang, Hongyuan Zha, Xiaofeng He,
  and Haifeng Chen.
\newblock Learning robust representations with graph denoising policy network.
\newblock In \emph{2019 IEEE International Conference on Data Mining (ICDM)},
  pages 1378--1383. IEEE, 2019{\natexlab{b}}.

\bibitem[Weiss et~al.(2018)Weiss, Goldberg, and Yahav]{weiss2018extracting}
Gail Weiss, Yoav Goldberg, and Eran Yahav.
\newblock Extracting automata from recurrent neural networks using queries and
  counterexamples.
\newblock In \emph{International Conference on Machine Learning}, pages
  5247--5256. PMLR, 2018.

\bibitem[Wilder et~al.(2019)Wilder, Ewing, Dilkina, and Tambe]{wilder2019end}
Bryan Wilder, Eric Ewing, Bistra Dilkina, and Milind Tambe.
\newblock End to end learning and optimization on graphs.
\newblock In \emph{Advances in Neural Information Processing Systems}, pages
  4674--4685, 2019.

\bibitem[Williams(1992)]{williams1992simple}
Ronald~J Williams.
\newblock Simple statistical gradient-following algorithms for connectionist
  reinforcement learning.
\newblock \emph{Machine learning}, 8\penalty0 (3-4):\penalty0 229--256, 1992.

\bibitem[Ying et~al.(2018)Ying, He, Chen, Eksombatchai, Hamilton, and
  Leskovec]{ying2018graph}
Rex Ying, Ruining He, Kaifeng Chen, Pong Eksombatchai, William~L Hamilton, and
  Jure Leskovec.
\newblock Graph convolutional neural networks for web-scale recommender
  systems.
\newblock In \emph{Proceedings of the 24th ACM SIGKDD International Conference
  on Knowledge Discovery \& Data Mining}, pages 974--983, 2018.

\bibitem[Yun et~al.(2019)Yun, Jeong, Kim, Kang, and Kim]{yun2019graph}
Seongjun Yun, Minbyul Jeong, Raehyun Kim, Jaewoo Kang, and Hyunwoo~J Kim.
\newblock Graph transformer networks.
\newblock In \emph{Advances in Neural Information Processing Systems}, pages
  11960--11970, 2019.

\bibitem[Zhang et~al.(2018)Zhang, Wang, Xiang, Huang, and
  Nehorai]{zhang2018retgk}
Zhen Zhang, Mianzhi Wang, Yijian Xiang, Yan Huang, and Arye Nehorai.
\newblock Retgk: Graph kernels based on return probabilities of random walks.
\newblock In \emph{Advances in Neural Information Processing Systems}, pages
  3964--3974, 2018.

\end{thebibliography}

\normalsize
\setcounter{prop}{0}

\appendix
\counterwithin{figure}{section}
\counterwithin{table}{section}
\section{Details on Constrained Reachability}
In this section we describe how program analyses can be converted to regular languages, and provide the proofs for the statements in section \ref{subsec:derived-relationships-as-constrained-reachability}.

\subsection{Example: Regular Languages For Program Analyses}\label{appendix:prog-analysis-regular}

The following grammar defines a regular language for the \scsym{LastWrite} edge type as described by \citet{allamanis2018learning}, where $n_2$ is the last write to variable $n_1$ if there is a path from $n_1$ to $n_2$ whose label sequence matches the nonterminal \textbf{Last-Write}. We denote node type labels with capitals, edge type labels in lowercase, and nonterminal symbols in boldface. For simplicity we assume a single target variable name with its own node type TargetVariable, and only consider a subset of possible AST nodes.

\newcommand{\grammarcomment}[1]{\emph{\color{black!60!white}{#1}}}
\begin{center}
\small
\begin{tabular}{rcl}
    &&\grammarcomment{All \scsym{LastWrite} edges start from a use of the target variable.}\\
    \textbf{Last-Write} &=& TargetVariable to-parent \textbf{Find-Current-Statement} \\
    \\
    &&\grammarcomment{Once we find a statement, go backward.}\\
    \textbf{Find-Current-Statement} &=& ExprStmt \textbf{Step-Backward} \\
        &\vline& Assign \textbf{Step-Backward} \\
    &&\grammarcomment{While in an expression, step out.}\\
        &\vline& BinOp to-parent \textbf{Find-Current-Statement} \\
        &\vline& Call to-parent \textbf{Find-Current-Statement} \\
    \\
    &&\grammarcomment{Stop if we find an assignment to the target variable.}\\
    \textbf{Check-Stmt} &=& Assign to-target TargetVariable\\
        &&\grammarcomment{Skip other statements.}\\
          &\vline& Assign to-target NonTargetVariable to-parent \textbf{Step-Backward} \\
          &\vline& ExprStmt \textbf{Step-Backward} \\
        &&\grammarcomment{Either enter If blocks or skip them.}\\
          &\vline& If to-last-child \textbf{Check-Stmt} \\
          &\vline& If \textbf{Step-Backward} \\
        &&\grammarcomment{Either enter While blocks or skip them, possibly jumping back to a break.}\\
          &\vline& While \textbf{Step-Backward} \\
          &\vline& While to-last-child \textbf{Check-Stmt} \\
          &\vline& While to-last-child \textbf{Find-Break-A} \\
          \\
        &&\grammarcomment{If we have a previous statement, check it.}\\
    \textbf{Step-Backward} &=& prev-stmt \textbf{Check-Stmt}\\
        &&\grammarcomment{If this is the first statement of an If block, exit.}\\
                  &\vline& from-first-child If \textbf{Step-Backward}\\
        &&\grammarcomment{If this is the first statement of a While block, either exit or go back}\\
        &&\grammarcomment{to the end of the loop body.}\\
                  &\vline& from-first-child While \textbf{Step-Backward}\\
                  &\vline& from-first-child While to-last-child \textbf{Check-Stmt}\\
          \\
          &&\grammarcomment{If we find a Break, this is a possible previous loop exit point.}\\
    \textbf{Find-Break-A} &=& Break \textbf{Step-Backward}\\
          &&\grammarcomment{Either way, keep looking for other break statements.}\\
            &\vline& Break \textbf{Find-Break-B}\\
                  &\vline& ExprStmt \textbf{Find-Break-B}\\
                  &\vline& If to-last-child \textbf{Find-Break-A} \\
          &&\grammarcomment{Don't enter while loops, since break statements only affect one loop.}\\
                  &\vline& While \textbf{Find-Break-B} \\
          \\
    \textbf{Find-Break-B} &=& prev-stmt \textbf{Find-Break-A}\\
                  &\vline& from-first-child If \textbf{Find-Break-B}\\
\end{tabular}
\end{center}

The constructions for \scsym{NextControlFlow} and \scsym{LastRead} are similar. Note that for \scsym{LastRead}, instead of skipping entire statements until finding an assignment, the path must iterate over all expressions used within each statement and check for uses of the variable. For \scsym{NextControlFlow}, instead of stepping backward, the path steps forward, and instead of searching for break statements after entering a loop, it searches for the containing loop when reaching a break statement. (This is because \scsym{NextControlFlow} simulates program execution forward instead of in reverse. Note that regular languages are closed under reversal \citep{hopcroft2001introduction}, so such a transformation between forward and reverse paths is possible in general; we could similarly construct a language for \scsym{LastControlFlow} if desired.)

\subsection{Proof of Proposition \ref{thm:equiv-mdp-language-APPENDIX}} \label{appendix:constrained-reachability-proof}

We recall proposition \ref{thm:equiv-mdp-language-APPENDIX}:

\begin{prop}\label{thm:equiv-mdp-language-APPENDIX}
Let $\mathcal{G}$ be a family of graphs annotated with node and edge types.
There exists
an encoding of graphs $G \in \mathcal{G}$ into POMDPs as described in section \ref{subsec:from-graphs-to-pomdps}
and a mapping from regular languages $L$ into finite-state policies $\pi_L$
such that,
for any $G \in \mathcal{G}$, there is an $L$-path from $n_0$ to $n_T$ in $G$ if and only if
$p(a_T = \scsym{AddEdgeAndStop}, n_T | n_0, \pi_L) > 0$.
\end{prop}
\begin{proof}
We start by defining a generic choice of POMDP conversion that depends only on the node and edge types.
Let $G \in \mathcal{G}$ be a directed graph with node types $\mathcal{N}$, edge types $\mathcal{E}$, nodes $N$, and edges $E \subseteq N \times N \times \mathcal{E}$. 
We convert it to a POMDP by choosing $\Omega_{\tau(n)} = \{(\tau(n), \scsym{TRUE}), (\tau(n), \scsym{FALSE})\}$, $\mathcal{M}_{\tau(n)} = \mathcal{E}$,
\[
p(n_{t+1} | n_t, a_t = (\scsym{Move}, m_t)) = \begin{cases}
1/|A^{m_t}_{n_t}| &\text{if }n_{t+1} \in A^{m_t}_{n_t},\\
1 &\text{if }n_{t+1} = n_t\text{ and }A^{m_t}_{n_t} = \emptyset,\\
0 &\text{otherwise},
\end{cases}
\]
$\omega_0 = (\tau(n_0), \scsym{TRUE})$, and $\omega_{t+1} = (\tau(n_{t+1}),n_{t+1} \in A^{m_t}_{n_t})$, where we let $A^{m_t}_{n_t} = \left\{ n_{t+1} ~\middle|~ (n_t, n_{t+1}, m_t) \in E \right\}$ be the set of neighbors adjacent to $n_t$ via an edge of type $m_t$.

Now suppose $L$ is a regular language over sequences of node and edge types. Construct a deterministic finite automaton $M$ that accepts exactly the words in $L$ (for instance, using the subset construction) \citep{hopcroft2001introduction}. Let $Q$ denote its state space, $q_0$ denote its initial state, $\delta : Q \times \Sigma \to Q$ be its transition function, and $F \subseteq Q$ be its set of accepting states.
We choose $Q$ as the finite state memory of our policy $\pi_L$, i.e. at each step $t$ we assume our agent is associated with a memory state $z_t \in Q$. We let $z_0 = q_0$ be the initial memory state of $\pi_L$.

Consider an arbitrary memory state $z_t \in Q$ and observation $\omega_t = (\tau(n_t), e_t)$. We now construct a set of possible next actions and memories $N_t \subseteq \mathcal{A}_t \times Q$. If $e_t = \scsym{FALSE}$, let $N_t = \varnothing$. Otherwise, let $z_{t+1/2} = \delta(z_t, \tau(n_t))$. If $z_{t+1/2} \in F$, add $(\scsym{AddEdgeAndStop}, z_{t+1/2})$ to $N_t$. Next, for each $m \in \mathcal{E}$, add $((\scsym{Move}, m), \delta(z_{t+1/2}, m))$ to $N_t$. Finally, let
\[
\pi_L(a_t, z_{t+1} | z_t, \omega_t) = \begin{cases}
1 / |N_t| &\text{if } (a_t, z_{t+1}) \in N_t,\\
1 &\text{if } N_t = \varnothing, a_t = \scsym{Stop}, z_{t+1} = z_t\\
0 &\text{otherwise.}
\end{cases}
\]

The $e_t = \scsym{FALSE} \implies N_t = \varnothing$ constraint ensures that the partial sequence of labels along any accepting trajectory matches the sequence of node type observations and movement actions produced by $\pi_L$. Since $\pi_L$ starts in the same state as $M$, and assigns nonzero probability to exactly the state transitions determined by $\delta$, it follows that the memory state of the agent along any partial trajectory $[n_0, m_0, n_1, m_1, \ldots, n_t]$ corresponds to the state of $M$ after processing the label sequence $[\tau(n_0), m_0, \tau(n_1), m_1, \ldots, \tau(n_t)]$.

Since $\pi_L$ assigns nonzero probability to the \scsym{AddEdgeAndStop} action exactly when memory state is an accepting state from $F$, and $M$ is in an accepting state from $F$ exactly when the label sequence is in $L$, we conclude that desired property holds.
\end{proof}

\begin{corollary*}
There exists an encoding of program AST graphs into POMDPs and a specific policy $\pi_{\textsc{\tiny NEXT-CF}}$ with finite-state memory such that $p(a_T = \scsym{AddEdgeAndStop}, n_T ~|~ n_0, \pi) > 0$ if and only if $(n_0, n_T)$ is an edge of type \scsym{NextControlFlow} in the augmented AST graph. Similarly, there are policies $\pi_{\textsc{\tiny LAST-READ}}$ and $\pi_{\textsc{\tiny LAST-WRITE}}$ for edges of type \scsym{LastRead} and \scsym{LastWrite}, respectively.
\end{corollary*}
\begin{proof}
This corollary follows directly from Proposition \ref{thm:equiv-mdp-language} and the existence of regular languages for these edge types (see appendix \ref{appendix:prog-analysis-regular}).
\end{proof}

Note that equivalent policies also exist for POMDPs encoded differently than the proof of proposition \ref{thm:equiv-mdp-language} describes. For instance, instead of having ``TargetVariable'' as a node type and constructing edges for each target variable name separately, we can extend the observation $\omega_t$ to contain information on whether the current variable name matches the initial variable name and then find all edges at once, which we do for our experiments. Additionally, if an action would cause the policy to transition into an absorbing but non-accepting state (i.e. a failure state) in the discrete finite automaton for $L$, the policy can immediately take a \scsym{Backtrack} or \scsym{Stop} action instead, or reallocate probability to other states, instead of just cycling forever in that state. This allows the policy $\pi$ to more evenly allocate probability across possible answers, and we observe that the GFSA policies learn to do this in our experiments.

\section{Graph to POMDP Dataset Encodings} \label{appendix:encoding-graphs-as-pomdps}
Here we describe the encodings of graphs as POMDPs that we use for our experiments.

\subsection{Python Abstract Syntax Trees}\label{appendix:encoding-graphs-ast}
We convert all of our code samples into the unified format defined by the \texttt{gast} library,\footnote{\url{https://github.com/serge-sans-paille/gast/releases/tag/0.3.3}} which is a slightly-modified version of the abstract syntax tree provided with Python 3.8 that is backward-compatible with older Python versions. We then use a generic mechanism to convert each AST node into one or more graph nodes and corresponding POMDP states.

Each AST node type $\tau$ (such as \texttt{FunctionDef}, \texttt{If}, \texttt{While}, or \texttt{Call}) has a fixed set $F$ of possible field names. We categorize these fields into four categories: optional fields $F_\text{opt}$, exactly-one-child fields $F_\text{one}$, nonempty sequence fields $F_\text{nseq}$, and possibly empty sequence fields $F_\text{eseq}$. We define the observation space at nodes of type $\tau$ as
\[
\Omega_\tau = \{\tau\} \times \Gamma \times \Psi_\tau
\]
where $\Gamma$ is a task-specific extra observation space, and $\Psi_\tau$ indicates the result of the previous action:
\begin{align*}
\Psi_\tau &= \{ (\scsym{FROM}, f) ~|~ f \in F \cup \{\scsym{PARENT}\} \} \cup \{ (\scsym{MISSING}, f) ~|~ f \in F_\text{opt} \cup F_\text{eseq}  \}.
\end{align*}
The $(\tau, \gamma, (\scsym{FROM}, f))$ observations are used when the agent moves to an edge of type $\tau$ from a child from field $f$ (or from the parent node), and the $(\tau, \gamma, (\scsym{MISSING}, f))$ observations are used when the agent attempts to move to a child for field $f$ but no such child exists. We define the movement space as
\begin{align*}
\mathcal{M}_\tau &= \{\scsym{GO-PARENT}\} \cup \{ (\scsym{GO}, f) ~|~ f \in F_\text{one} \cup F_\text{opt} \} \}
\\&\qquad\cup \{ (x, f) ~|~ x \in \{\scsym{GO-FIRST}, \scsym{GO-LAST}, \scsym{GO-ALL}\},\, f \in F_\text{nseq} \cup F_\text{eseq}  \}.
\end{align*}
$\scsym{GO}$ moves the agent to the single child for that field, $\scsym{GO-FIRST}$ moves it to the first child, $\scsym{GO-LAST}$ moves it to the last child, and $\scsym{GO-ALL}$ distributes probability evenly among all children. $\scsym{GO-PARENT}$ moves the agent to the parent node; we omit this movement action for the root node ($\tau = \text{Module}$).

For each sequence field $f \in F_\text{nseq} \cup F_\text{eseq}$, we also define a helper node type $\tau_f$, which is used to construct a linked list of children. This helper node has the fixed observation space
\begin{align*}
\Psi_{\tau_f} = \{&
\scsym{FROM-PARENT},
\scsym{FROM-ITEM},
\scsym{FROM-NEXT},
\scsym{FROM-PREV},
\\&\scsym{MISSING-NEXT},
\scsym{MISSING-PREV}
\}
\end{align*}
and action space
\begin{align*}
\mathcal{M}_{\tau_f} = \{&
\scsym{GO-PARENT},
\scsym{GO-ITEM},
\scsym{GO-NEXT},
\scsym{GO-PREV}
\}.
\end{align*}
When encoding the AST as a graph, helper nodes of this type are inserted between the AST node of type $\tau$ and the children for field $f$: the ``parent'' of a helper node is the original AST node, and the ``item'' of the $n$th helper node is the $n$th child of the original AST node for field $f$.

We note that this construction differs from the construction in the proof of proposition \ref{thm:equiv-mdp-language}, in that movement actions are specific to the node type of the current node. When the agent takes the \textsc{GO-PARENT} action, the observation for the next step informs it what field type it came from. This helps keep the state space of the GFSA policy small, since it does not have to guess what its parent node is and then remember the results; it can instead simply walk to the parent node and then condition its next action on the observed field. The construction described here still allows encoding the edges from \scsym{NextControlFlow}, \scsym{LastRead}, and \scsym{LastWrite} as policies, as we empirically demonstrate by training the GFSA layer to replicate those edges.

\subsection{Grid-world Environments}
For the grid-world environments, we represent each traversable grid cell as a node, and classify the cells into eleven node types corresponding to which movement directions (left, up, right, and down) are possible:
\[
\mathcal{N} = \{\text{LU}, \text{LR}, \text{LD}, \text{UR}, \text{UD}, \text{RD}, \text{LUR}, \text{LUD}, \text{LRD}, \text{URD}, \text{LURD}\}
\]
Note that in our dataset, no cell has fewer than two neighbors.

For each node type $\tau \in \mathcal{N}$ the movement actions $\mathcal{M}_\tau$ correspond exactly to the possible directions of movement; for instance, cells of type LD have $\mathcal{M}_\text{LD} = \{\text{L}, \text{D}\}$. We use a trivial observation space $\Omega_\tau = \{\tau\}$, i.e. the GFSA automaton sees the type of the current node but no other information.

When converting grid-world environments into POMDPs, we remove the \scsym{Backtrack} action to encourage the GFSA edges to match more traditional RL option sub-policies. 

\section{GFSA Layer Implementation}\label{appendix:gfsa-impl}

Here we describe additional details about the implementation of the GFSA layer.

\subsection{Parameters}\label{appendix:subsec-gfsa-impl-parameters}

We represent the parameters $\theta$ of the GFSA layer as a table indexed by feasible observation and action pairs $\Phi$ as well as state transitions:
\begin{align*}
\Phi &= \left(\bigcup_{\tau \in \mathcal{N}} \Omega_\tau \times \mathcal{A}_\tau\right),
&
\theta &: Z \times \Phi \times Z \to \mathbb{R},
\end{align*}
where $Z = \{0, 1, \ldots, |Z|-1\}$ is the set of memory states. We treat the elements of $\theta$ as unnormalized log-probabilities and then set $\pi = \softmax(\theta)$, normalizing separately across actions and new memory states for each possible current memory state and observation.

To initialize $\theta$, we start by defining a ``base distribution'' $p$, which chooses a movement action at random with probability 0.95 and a special action (\scsym{AddEdgeAndStop}, \scsym{Stop}, \scsym{Backtrack}) otherwise, and which stays in the same state with probability 0.8 and changes states randomly otherwise. Next, we sample our initial probabilities $q$ from a Dirichlet distribution centered on $p$ (with concentration parameters $\alpha_i = p_i / \beta$ where $\beta$ is a temperature parameter), and then take a (stabilized) logarithm $\theta_i = \log(q_i + 0.001)$. This ensures that the initial policy has some initial variation, while still biasing it toward staying in the same state and taking movement actions most of the time.

\subsection{Algorithmic Details}

As a preprocessing step, for each graph in the dataset, we compute the set $X$ of all $(n, \omega)$ node-observation pairs for the corresponding MDP. We then compute ``template transition matrices'', which specify how to convert the probability table $\theta$ into a transition matrix by associating transitions $X \times X$ and halting actions $X \times \left\{ \scsym{AddEdgeAndStop}, \scsym{Stop}, \scsym{Backtrack} \right\}$ with their appropriate indicies into $\Phi$. Then, when running the model, we retrieve blocks of $\theta$ according to those indices to construct the transition matrix for that graph (implemented with ``gather'' and ``scatter'' operations).

Conceptually, each possible starting node $n_0$ could produce a separate transition matrix $Q_{n_0} : X \times Z \times X \times Z \to \mathbb{R}$ because part of the observation in each state (which we denote $\gamma \in \Gamma$ and leave out of $X$) may depend on the starting node or other learned parameters. We address this by instead computing an ``observation-conditioned'' transition tensor
\[
Q : \Gamma \times X \times Z \times X \times Z \to \mathbb{R}
\]
that specifies transition probabilities for each observation $\gamma$,
along with a start-node-conditioned observation tensor
\[
C : N \times X \times \Gamma \to \mathbb{R}
\]
that specifies the probability of observing $\gamma$ for a given start node $n_0$ and current $(n, \omega)$ tuple. In order to compute a matrix-vector product $Q_{n_0} \bvec{v}$ we can then use the tensor product
\[
\sum_{i, z, \gamma}\, C_{n_0, i, \gamma}\, Q_{\gamma, i, z, i', z'}\, v_{i,z}
\]
which can be computed efficiently without having to materialize a separate transition matrix for every start node $n_0 \in N$.

During the forward pass through the GFSA layer, to solve for the absorbing probabilities in equation \ref{eqn:transition-solve-forward}, we iterate
\begin{align*}
\bvec{x}_0 &= \bvec{\delta}_{n_0},
&
\bvec{x}_{k+1} = \bvec{\delta}_{n_0} + Q_{n_0} \bvec{x}_k
\end{align*}
until a fixed number of steps $K$, then approximate
\[
p(a_T, n_T | n_0, \pi)
= H_{(a_T, n_T), :} \left(I - Q_{n_0}\right)^{-1} \bvec\delta_{n_0}
\approx H_{(a_T, n_T), :} \,\bvec{x}_{K}.
\]
To efficiently compute the backwards pass without saving all of the values of $\bvec{x}_k$, we use the \texttt{jax.lax.custom\_linear\_solve} function from JAX \citep{jax2018github}, which converts the gradient equations into a transposed matrix system
\[
\left(I-Q_{n_0}^T\right)^{-1} H^T
\frac{\partial{\mathcal{L}}}{\partial p( \cdot | n_0, \pi)} 
\]
that we similarly approximate with
\begin{align*}
\bvec{y}_0 &= H^T
\frac{\partial{\mathcal{L}}}{\partial p( \cdot | n_0, \pi)},
&
\bvec{y}_{k+1} &= H^T
\frac{\partial{\mathcal{L}}}{\partial p( \cdot | n_0, \pi)} + Q_{n_0}^T \bvec{y}_k.
\end{align*}
Note that both iteration procedures are guaranteed to converge because the matrix $I - Q_{n_0}$ is diagonally dominant \citep{saad2003iterative}.
Conveniently, we implement $Q_{n_0} \bvec{x}_k$ using the tensor product described above, and JAX automatically translates this into a computation of the transposed matrix-vector product $Q_{n_0}^T \bvec{y}_k$ using automatic differentiation.

If the GFSA policy assigns a very large probability to the \scsym{Backtrack} action, this can lead to numerical instability when computing the final adjacency matrix, since we condition on non-backtracking trajectories when computing our final adjacency matrix. We circumvent this issue by constraining the policy such that a small fraction of the time ($\varepsilon_\text{bt-stop}$), if it attempts to take the \scsym{Backtrack} action, it instead takes the \scsym{Stop} action; this ensures that, if the policy backtracks with high probability, the weight of the produced edges will be low. Additionally, we attempt to mitigate floating-point precision issues during normalization by summing over \scsym{AddEdgeAndStop} and \scsym{Stop} actions instead of computing $1 - p(a_t = \scsym{Backtrack} | \cdots)$ directly.

When computing an adjacency matrix from the outputs of the GFSA layer, there are two ways to extract multiple edge types. The first is to associate each edge type with a distinct starting state in $Z$. The second way is to compute a different version of the parameter vector $\theta$ for each edge type. For the variable misuse experiments, we use the first method, since sharing states uses less memory. For the grid-world experiments, we use the second method, as we found that using non-shared states gives slightly better performance and results in more interpretable learned options.

\subsection{Asymptotic Complexity}

We now give a brief complexity analysis for the GFSA layer implementation. Let $n$ be the number of nodes, $e$ be the number of edges, $z$ be the number of memory states, $\omega$ be the number of ``static'' observations per node (such as \scsym{FROM-PARENT}), and $\gamma$ be the number of ``dynamic'' observations per node (for instance, observations conditioned on learned node embeddings).

\textbf{Memory:} Storing the probability of visiting each state in $X \times Z$ for a given source node takes memory $O(n \omega z)$, so storing it for all source nodes takes $O(n^2 \omega z)$.
For a dense representation of $Q$ and $C$, $Q$ takes $O(n^2 \omega^2 z^2 \gamma)$ memory and $C$ takes $O(n^2 \omega \gamma)$. Since memory usage is independent of the number of iterations, the overall memory cost is thus $O(n^2 \omega^2 z^2 \gamma)$. We note that memory scales proportional to the square of the number of nodes, but this is in a sense unavoidable since the output of the GFSA layer is a dense $n \times n$ matrix even if $Q$ is sparse.

\textbf{Time:} Computing the tensor product for all starting nodes requires computing $n \times n \omega z \times n \omega z \times \gamma$ elements. Since we do this at every iteration, we end up with a time cost of $O(T_{\max{}} n^3 \omega^2 z^2 \gamma)$. We note that a sparse representation of $Q$ might reduce this to $O(T_{\max{}} (n^2 z + n e z^2) \omega \gamma)$ (since we could first contract with C with cost $n \times n \omega \times z \times \gamma$ and then iterate over edges, start nodes, observations, and states, with cost $n \times e \times \omega z\times z \times \gamma$). However, in practice we use a dense implementation to take advantage of fast accelerator hardware.

\section{Experiments, Hyperparameters, and Detailed Results}

Here we describe additional details for each of our experiments and the corresponding evaluation results. For all of our experiments, we train and evaluate on TPU v2 accelerators.\footnote{\url{https://cloud.google.com/tpu/}} Each training job uses 8 TPU v2 cores, evenly dividing the batch size between the cores and averaging gradients across them at each step.

\subsection{Grid-world Task} \label{appendix:gridworld}
We use the LabMaze generator (\url{https://github.com/deepmind/labmaze}) from DeepMind Lab \citep{beattie2016deepmind} to generate our grid-world layouts. We configure it with a width and height of 19 cells, a maximum of 6 rooms, and room sizes between 3 and 6 cells on each axis. We then convert the generated grids into graphs, and filter out examples with more than 256 nodes or 512 node-observation tuples. We generate 100,000 training graphs and 100,000 validation graphs. For each graph, we then pick 32 goal locations uniformly at random.

We configure the GFSA layer to use four independent policies to produce four derived edge types. For each policy, we set the memory space to the two-element set $Z = \{0, 1\}$, where $z_0 = 0$. We initialize parameters using temperature $\beta = 0.2$, but use the Dirichlet sample directly as a logit, i.e. $\theta_i = q_i$. (We found that applying the logarithm from appendix \ref{appendix:gfsa-impl} during initialization yields similar numerical performance but makes the learned policies harder to visualize.)

Since the interpretation of the edges as options requires them to be properly normalized (i.e. the distribution $p(s_{t+1} | s_t, a_t)$ must be well defined), we make a few modifications to the output adjacency matrix produced by the GFSA layer. In particular, we do not use the learned adjustment parameters described in section \ref{subsec:derived-adjacency-matrix}, instead fixing $a = 1, b = 0$. We also ensure that the edge weights are normalized to 1 for each source node by assigning any missing mass to the diagonal. In other words, if the GFSA sub-policy agent takes a \scsym{Stop} action or fails to take the \scsym{AddEdgeAndStop} action before $T_{\max}$ iterations, we instead treat the option as a no-op that causes the primary agent to remain in place. We also remove the \scsym{Backtrack} action.

At each training iteration, we sample a graph $G$ from our training set, and in parallel compute approximate entropy-regularized optimal policies $\pi^*$ for each of the 32 goal locations for $G$.
Mathematically, for each goal $g$, we seek
\[
\pi^* = \argmax_\pi\, \mathbb{E}_{(s_t, a_t) \sim p(\cdot ~|~ \pi)}\left[ \sum_{0 \le t \le T} -1 + \mathcal{H}(\pi(\cdot ~|~ s_t)) ~\middle|~ s_T = g \right],
\]
where $\mathcal{H}(\pi(\cdot ~|~ s_t))$ denotes the entropy of the distribution over actions, and we have fixed the reward to -1 for all timesteps. We use an entropy-regularized objective here so that the policy $\pi^*$ is nondeterministic and thus has a useful derivative with respect to the option distribution. As described by \citet{haarnoja2017reinforcement}, we can compute this optimal policy by doing soft-Q iteration using the update equations
\begin{align*}
    Q_\text{soft}(s_t, a_t)
    &\leftarrow \E_{s_{t+1} \sim p(\cdot | s_t, a_t)}\left[ V_\text{soft}(s_{t+1}) \right] - 1,
    \\
    V_\text{soft}(s_t) &\leftarrow \log \sum_{a_t} \exp \left(  Q_\text{soft}(s_t, a_t) \right).
\end{align*}
until reaching a fixed point, and then letting
\[
\pi^*(a_t | s_t) = \exp(Q_\text{soft}(s_t, a_t) - V_\text{soft}(s_t)).
\]
Since our graphs are small, we can store $Q_\text{soft}$ and $V_\text{soft}$ in tabular form, and directly solve for their optimal values by iterating the above equations. In practice, we approximate the solution by using 512 iterations.

After computing $Q^{(g)}_\text{soft}, V^{(g)}_\text{soft},$ and $\pi_{(g)}^*$ for each choice of $g$, we then define a minimization objective for the full task as
\[
\mathcal{L} = - \E_{s_0, g}\left[ V^{(g)}_\text{soft}(s_0) \right],
\]
i.e. we seek to maximize the soft value function across randomly chosen sources and goals, or equivalently to minimize the expected number of steps taken by $\pi_{(g)}^*$ before reaching the goal. We compute gradients by using implicit differentiation twice: first to differentiate through the fixed point to the soft-Q iteration, and second to differentiate through the computation of the GFSA edges. Implicitly differentiating through the soft-Q equations is conceptually similar to implicit MAML \citep{rajeswaran2019meta} except that the parameters we optimize in the outer loop (the GFSA parameters) are not the same as the parameters we optimize in the inner loop (the graph-specific tabular policy).

Note that differentiating through the soft-Q fixed point requires first linearizing the equations around the fixed point. More specifically, if we express the fixed point equations in terms of a function $V_\text{soft} = f(V_\text{soft}, \theta)$ where $\theta$ represents the GFSA parameters, we have
\[
\partial V_\text{soft} = \partial f(V_\text{soft}, \theta) = f_{V}(V_\text{soft}, \theta) \partial V_\text{soft} + f_{\theta}(V_\text{soft}, \theta) \partial \theta
\]
(where $f_{V}(V_\text{soft}, \theta)$ denotes the Jacobian of $f$ with respect to $V$, and similarly for $f_\theta$) and thus
\[
\partial V_\text{soft} = \big(I - f_{V}(V_\text{soft}, \theta)\big)^{-1} f_{\theta}(V_\text{soft}, \theta) \partial \theta
\]
which leads to gradient equations
\[
\frac{\partial\mathcal{L}}{\partial \theta} = f_{\theta}(V_\text{soft}, \theta)^T \big(I - f_{V}(V_\text{soft}, \theta)^T\big)^{-1} \frac{\partial\mathcal{L}}{\partial{V_\text{soft}}}.
\]
As before, JAX makes it possible to easily express these gradient computations and automatically handles the computation of the relevant partial derivatives and Jacobians. In this case, due to the small number of goal locations and lack of diagonal dominance guarantees, we simply compute and invert the matrix $I - f_{V}(V_\text{soft}, \theta)^T$ during the backward pass instead of using an iterative solver. (See \citet{liao2018reviving} for additional information about implicitly differentiating through fixed points.)

We trained the model using the Adam optimizer, with a learning rate of 0.001 and a batch size of 32 graphs with 32 goals each, for approximately 50,000 iterations, until the validation loss plateaued. We then picked a grid from the validation set, and chose four possible starting locations manually to give a summary of the overall learned behavior.

\subsection{Static Analyses} \label{appendix:edge-classification}

\subsubsection{Datasets}

For the static analysis tasks, we first generate a dataset of random Python programs using a probabilistic context free grammar. This grammar contains a variety of nonterminals and associated production rules:

\begin{itemize}
    \item \textbf{Number}: An integer or float expression. Either a variable dereference, a constant integer, an arithmetic operation, or a function call with numeric arguments.
    \item \textbf{Boolean}: A boolean expression. Either a comparison between numbers, a constant \texttt{True} or \texttt{False}, or a boolean combination using \texttt{and} or \texttt{or}.
    \item \textbf{Statement}: A single statement. Either an assignment, a call to \texttt{print}, an if, if-else, for, or while block, or a \texttt{pass} statement.
    \item \textbf{Block}: A contiguous sequence of statements that may end in a \texttt{return}, \texttt{break}, or \texttt{continue}, or with a normal statement; we only allow these statements at the end of a block to avoid producing dead code.
\end{itemize}

\begin{figure}
\centering
\small
\begin{center}
\begin{BVerbatim}
def generated_function(a, b):
    for v2 in range(int(bar_2(a, b))):
        v3 = foo_4(v2, b, bar_1(b), a) / 42
        v3 = (b + 8) * foo_1((v2 * v3))
        pass
        v3 = b
        while False:
            a = v3
            a = v3
            v2 = 34
            break
        if bar_1((b * b)) != v2:
            v4 = foo_4(bar_2(56, bar_1(v2)),
                       foo_4(b, a, a, 39) - v2, 
                       bar_1(a), 32)
            for v5 in range(int(v4)):
                v6 = v4
                pass
                break
            print(69)
        v2 = v2
    b = 15
    b = ((a + 96) + 89) - a
    v2 = foo_4(b, 21, 26, foo_4(85, a, a - b, a))
    b = v2 - (a - v2)
\end{BVerbatim}
\end{center}
\caption{Example of a program from the ``1x'' program distribution.}
\label{fig:appendix:prog-dist-1x}
\end{figure}

\begin{figure}
\centering
\small
\begin{center}
\begin{BVerbatim}
def generated_function(a, b):
    if bar_1(b) > b:
        b = a
        print(b)
    else:
        a = a + a
    a = bar_1(62 - 35)
    if b <= bar_1(54):
        b = b
        while a >= 58:
            b = foo_1(a)
            pass
            pass
    else:
        a = bar_4(b, bar_1(b), bar_1(a * a), bar_1(a))
    b = 88
\end{BVerbatim}
\end{center}
\caption{Example of a program from the ``0.5x'' program distribution.}
\label{fig:appendix:prog-dist-0pt5x}
\end{figure}

\begin{figure}
\centering
\scriptsize
\begin{center}
\begin{BVerbatim}
def generated_function(a, b):
    v2 = b
    pass
    b = v2
    pass
    v2 = b
    b = bar_1(v2)
    v3 = v2
    print(b)
    b = bar_1(v3) + (bar_1(20) - v2)
    print(56)
    if (foo_2(v2 + v3, foo_1(a)) == a) or ((foo_2(b, v2) < 22) or (v2 >= a or 37 <= v3)):
        v3 = b
        print(v2)
        print(foo_1(a))
        a = v3
        b = foo_1(a)
        v4 = foo_1(b)
        print(foo_1(v3) * bar_1(v2))
        b = a
        print(bar_1(v2))
        v2 = v2
        v4 = 67
        v5 = bar_2((v4 + v2) / (a / b), b)
    else:
        v4 = v3
        b = v4
        while ((v2 + (a / v4)) * (foo_2(93, v2) + v2)) < ((a - v2) - 18):
            v5 = v2
            v6 = bar_2(v3, (a - v4) + v3)
            a = v2 + v2
            break
        v5 = 71 / v2
        v6 = (a + (b + 47)) - (foo_2(v2, a) / (v5 * v3))
        b = v5
        b = foo_2(v3, v4)
        v5 = 14
        v3 = v3
    v4 = b * foo_1(b)
    v5 = bar_4(v2, v3, v4, b)
    v6 = a
    v3 = v5
    b = 11
    v7 = foo_1(v2)
    v8 = v4
    v4 = foo_1(foo_1(b))
    v4 = bar_1(bar_1(bar_2(32, v5)))
    v5 = bar_1(bar_2(v2, v2))
\end{BVerbatim}
\end{center}
\caption{Example of a program from the ``2x'' program distribution.}
\label{fig:appendix:prog-dist-2x}
\end{figure}

We apply constraints to the generation process such that variable names are only used after they have been defined, expressions are limited to a maximum depth, and statements continue to be generated until reaching a target number of AST nodes. For the training dataset, we set this target number of nodes to 150, and convert each generated AST into a graph according to \ref{appendix:encoding-graphs-ast}; we then throw out graphs with more than 256 graph nodes or 512 node-observation tuples. For our test datasets, we use a target AST size of 300 AST nodes and cutoffs of 512 graph nodes or 1024 node-observation tuples for the ``2x'' dataset, and a target of 75 AST nodes and cutoffs of 128 graph nodes and 512 tuples for the ``0.5x'' dataset. For each dataset, the graph size cutoff results in keeping approximately 95\% of the generated ASTs. Figures \ref{fig:appendix:prog-dist-1x},  \ref{fig:appendix:prog-dist-0pt5x}, and \ref{fig:appendix:prog-dist-2x} show example programs from these distributions.

We generated a training dataset with 100,000 programs, a validation dataset of 1024 programs, and a test dataset of 100,000 programs for each of the three sizes (1x, 0.5x, 2x).

\subsubsection{Architectures and Hyperparameters}\label{appendix:subsec:edge-clf-architecture}

We configure the GFSA layer to produce a single edge type, corresponding to the target edge of interest. For this task, we specify the the task-specific observation $\gamma$ referenced in appendix \ref{appendix:encoding-graphs-ast} such that the agent can observe when its current node is a variable with the same identifier as the initial node.  We treat $|Z|$ as a hyperparameter, varying between 2, 4, and 8, with a fixed starting state $z_0$. We additionally randomly sample the backtracking stability hyperparameter $\varepsilon_\text{bt-stop}$ according to a log-uniform distribution within the range $[0.001, 0.1]$ (see appendix \ref{appendix:gfsa-impl}). We initialize parameters with temperature $\beta = 0.01$. Since we choose an optimal threshold while computing the F1 score, we do not use the learned adjustment parameters described in section \ref{subsec:derived-adjacency-matrix}, and instead fix $a = 1, b=0$.

For the GGNN, GREAT, and RAT baselines, we evaluate with both ``nodewise" and ``dot-product" heads. For the "nodewise" head, we compute outputs as
\[
A_{n, n'} = \sigma\big(\big[f_\theta(X_\text{node} + \bvec{b}^T \bvec{\delta}_{n}, X_\text{edge})\big]_{n'}\big)
\]
where the learned model $f_\theta : \R^{d \times |N|} \times \R^{e \times |N| \times |N|} \to \R^{|N|}$ produces a scalar output for each node, $X_\text{node} \in \R^{d \times |N|}$ and $X_\text{edge} \in \R^{e \times |N| \times |N|}$ are embeddings of the node and edge features, $\bvec{\delta}_{n}$ is a one-hot vector indicating the start node, and $\bvec{b}$ is a learned start node embedding. For the ``dot-product'' head, we instead compute
\begin{align*}
Y &= f_\theta(X_\text{node}, X_\text{edge}),
&
A_{n, n'} = \sigma\big( \bvec{y}_{n}^T W \bvec{y}_{n'} + b\big),
\end{align*}
where the learned model $f_\theta : \R^{d \times |N|} \times \R^{e \times |N| \times |N|} \to \R^{d \times |N|}$ produces updated node embeddings $\bvec{y}_{n}$, $W$ is a learned $d \times d$ matrix, and $b$ is a learned scalar bias. Since the nodewise models require $|N|$ times as many more forward passes to compute edges for a single example, we keep training time manageable by reducing the width relative to the dot-product models.

The RAT and GREAT models are both variants of a transformer applied to the nodes of a graph. Both models use a set of attention heads, each of which compute query and key vectors $\bvec{q}_n, \bvec{k}_n \in \R^d$ for each node $n$ as linear transformations of the node features $\bvec{x}_n$:  $\bvec{q}_n = W^Q \bvec{x}_n, \bvec{k}_n = W^K \bvec{x}_n$. The RAT model computes attention logits as
\[
y_{(n, n')} = \frac{\bvec{q}_n^T \big( \bvec{k}_{n'} + W^{EK} \bvec{e}_{(n, n')} \big)}{\sqrt{d}}
\]
where we transform the edge features $\bvec{e}_{(n, n')}$ into an ``edge key'' that can be attended to by the query in addition to the content-based key. This corresponds to the attention equations as described by \citet{shaw2018self}, but with a graph-based mechanism for choosing the pairwise key vector. The GREAT model uses an easier-to-compute formulation
\[
y_{(n, n')} = \frac{\bvec{q}_n^T \bvec{k}_{n'} + \bvec{w}^T \bvec{e}_{(n, n')} \cdot \mathbf{1}^T \bvec{k}_{n'}}{\sqrt{d}}
\]
where the attention logits are biased by a (learned) linear projection of the edge features, scaled by a (fixed) linear projection of the key ($\mathbf{1}$ denotes a vector of ones). In both models, the $y_{(n, n')}$ are converted to attention weights $\alpha_{(n, n')}$ using softmax, and used to compute a weighed average of embedded values. However, in the RAT model, both nodes and edges contribute to values ($z_n = \sum_{n'} \alpha_{(n, n')} (\bvec{v}_{n'} + W^{EV} \bvec{e}_{(n, n')})$), whereas in GREAT this sum is only over nodes ($z_n = \sum_{n'} \alpha_{(n, n')} \bvec{v}_{n'}$).

For the NRI-encoder-based model, we make multiple adjustments to the formulation from \citet{kipf2018neural} in order to apply it to our setting. Since we are adding edges to an existing graph, the first part of our NRI model combines aspects from the encoder and decoder described in \citet{kipf2018neural}; we express our version in terms of blocks that each compute
\begin{align*}
\bvec{h}^{i+1}_{(n, n')} &= \sum_{k} e_{k,(n,n')} f_e^{i, k}(\bvec{h}^{i}_{n}, \bvec{h}^{i}_{n'}),
&
\bvec{h}^{i+1}_{n} &= f^i_v\left(\sum_{n'} \bvec{h}^{i+1}_{(n, n')} \right).
\end{align*}
where $\bvec{h}^{i}_{n}$ denotes the vector of node features after layer $i$, $\bvec{h}^{i+1}_{(n, n')}$ denotes the vector of hidden pairwise features, and $e_{k,(n,n')}$ is the $k$th edge feature between $n$ and $n'$ from the base graph. To enable deeper models, we apply layer normalization and residual connections after each of these blocks, as in \citet{vaswani2017attention}. We then compute the final output head by applying the sigmoid activation to the final layer's hidden pairwise feature matrix $\bvec{h}^{I}_{(n, n')}$ (which we constrain to have feature dimension 1), replacing the softmax used in the original NRI encoder (since we are doing binary classification, not computing a categorical latent variable). All versions of $f$ are learned MLPs with ReLU activations.

The RL agent baseline uses the same parameterization as the GFSA layer. However, instead of exactly solving for marginals, we sample a discrete transition at every step. Given a particular start node, the agent gets a reward of +1 if it takes the \scsym{AddEdgeAndStop} action at any of the correct destination nodes, or if it takes the \scsym{Stop} action and there was no correct destination node. We use 20 rollouts per start node, and train with REINFORCE and a leave-one-out control variate. During final evaluation, we compute exact marginals as for the GFSA layer; thus, differences in evaluation results reflect differences in the learning algorithm only.

For all of our baselines, we convert the Python AST into a graph by transforming the AST nodes into graph nodes and the field relationships into edges. For parity with the GFSA layer, the helper nodes defined in appendix \ref{appendix:encoding-graphs-ast} are also used in the the baseline graph representation, and we add an extra edge type connecting variables that use the same identifier. All edges are embedded in both forward and reverse directions.
We include hyperparameters for whether the initial node embeddings $X_\text{node}$ contain positional encodings computed as in \citet{vaswani2017attention} according to a depth-first tree traversal, and whether edges are embedded using a learned vector or using a one-hot encoding.

For the GGNN model, we choose a number of GGNN iterations (between 4 and 12 iterations using the same parameters) and a hidden state dimension (from $\{16, 32, 128\}$ for the nodewise models or $\{128, 256, 512\}$ for the dot-product models).

For the GREAT and RAT models, we choose a number of layers (between 4 and 12, but not sharing parameters), a hidden state dimension (from $\{16, 32, 128\}$ for the nodewise models or $\{128, 256, 512\}$ for the dot-product models), and a number of self-attention heads (from $\{2, 4, 8, 16\}$), with query, key, and value sizes chosen so that the sum of sizes across all heads matches the hidden state dimension.

For the NRI encoder model, we choose whether to allow communication between non-adjacent nodes, a hidden size for node features (from $\{128, 256, 512\}$), a hidden size for intermediate pairwise features (from $\{16, 32, 64\}$), a hidden size for initial base-graph edge features (from $\{16, 32, 64\}$), a depth for each MLP (from 1 to 5 layers), and a number of NRI-style blocks (between 4 and 12).

\subsubsection{Training and Detailed Results}\label{appendix:edge-classification-detailed-results}

For all of our models, we train using the Adam optimizer for either 500,000 iterations or 24 hours, whichever comes first; this is enough time for all models to converge to their final accuracy. For each model version and task, we randomly sample 32 hyperparameter settings, and then select the model and early-stopping point with the best F1 score on a validation set of 1024 functions. In addition to the hyperparameters described above, all models share the following hyperparameters: batch size (either 8, 32, or 128), learning rate (log-uniform in $[10^{-5}, 10^{-2}]$), gradient clipping threshold (log-uniform in $[1, 10^4]$), and focal-loss temperature $\gamma$ (uniform in $[0, 5]$). Hyperparameter settings that result in out-of-memory errors are not counted toward the 32 samples.

After selecting the best performing model for each model type and task based on performance on the validation set, we evaluated the model on each of our test datasets. For each example size (1x, 2x, 0.5x), we partitioned the 100,000 test examples into 10 equally-sized folds. We used the first fold to tune the final classifier threshold to maximize F1 score (using a different threshold for each example size to account for shifts in the distribution of model outputs). We then fixed that threshold and evaluated the F1 score on each of the other splits. We report the mean of the F1 score across those folds, along with an approximate standard error estimate (computed by dividing the standard deviation of the F1 score across folds by $\sqrt{9} = 3$).

To assess robustness of convergence, we also compute the fraction of training runs that achieve at least 90\% accuracy on the validation set. Note that each training job has different hyperparameters but also a different parameter initialization and a different dataset iteration order; we do not attempt to distinguish between these sources of variation.

\begin{table}[hp]
\centering
\caption{Full-precision results on static analysis tasks. Expressed as mean F1 score (in \%) $\pm$ standard error on test set. For 1x dataset size, we also report fraction of training jobs across hyperparameter sweep that achieved 90\% validation accuracy.}
\label{appendix:tab:edge_classification}
\ssmall
\vspace{1em}
\begin{tabular}{rrrr}
\toprule
\textbf{Task} & \multicolumn{3}{c}{Next Control Flow}  \\
\cmidrule(r{4pt}){2-4}
\textbf{Example size} &
\multicolumn{1}{c}{1x}&
\multicolumn{1}{c}{2x}&
\multicolumn{1}{c}{0.5x}\\
\midrule
\multicolumn{4}{c}{\textbf{100,000 training examples}}\\
\midrule

\emph{RAT nw} &
99.9837 $\pm$ 0.0006~~(25/32 @ 90\%) & 99.9367 $\pm$ 0.0012 & 99.9880 $\pm$ 0.0007 \\
\emph{GREAT nw} &
99.9770 $\pm$ 0.0011~~(26/32 @ 90\%) & 99.8709 $\pm$ 0.0013 & 99.9834 $\pm$ 0.0010 \\
\emph{GGNN nw} &
99.9823 $\pm$ 0.0007~~(31/32 @ 90\%) & 93.9034 $\pm$ 0.0304 & 97.7723 $\pm$ 0.0246 \\
\emph{RAT dp} &
99.9945 $\pm$ 0.0004~~(26/32 @ 90\%) & 92.5278 $\pm$ 0.0080 & 96.5901 $\pm$ 0.0150 \\
\emph{GREAT dp} &
99.9941 $\pm$ 0.0006~~(24/32 @ 90\%) & 96.3243 $\pm$ 0.0092 & 98.3557 $\pm$ 0.0081 \\
\emph{GGNN dp} &
99.9392 $\pm$ 0.0014~~(26/32 @ 90\%) & 62.7524 $\pm$ 0.0195 & 98.5104 $\pm$ 0.0176 \\
\emph{NRI encoder} &
99.9765 $\pm$ 0.0010~~(31/32 @ 90\%) & 85.9087 $\pm$ 0.0156 & 99.9161 $\pm$ 0.0021 \\
\emph{RL ablation} &
94.2419 $\pm$ 0.0118~~(02/32 @ 90\%) & 93.5616 $\pm$ 0.0087 & 94.8329 $\pm$ 0.0241 \\
\emph{GFSA Layer (ours)} &
\textbf{99.9972 $\pm$ 0.0001}~~(29/32 @ 90\%) & \textbf{99.9941 $\pm$ 0.0002} & \textbf{99.9985 $\pm$ 0.0002} \\

\midrule
\multicolumn{4}{c}{\textbf{100 training examples}}\\
\midrule

\emph{RAT nw} &
98.6324 $\pm$ 0.0090~~(13/32 @ 90\%) & 95.9320 $\pm$ 0.0092 & 96.3167 $\pm$ 0.0249 \\
\emph{GREAT nw} &
98.2327 $\pm$ 0.0054~~(13/32 @ 90\%) & 97.9814 $\pm$ 0.0071 & 98.5181 $\pm$ 0.0065 \\
\emph{GGNN nw} &
99.3749 $\pm$ 0.0060~~(25/32 @ 90\%) & 98.3590 $\pm$ 0.0050 & 98.6022 $\pm$ 0.0141 \\
\emph{RAT dp} &
81.8068 $\pm$ 0.0296~~(00/32 @ 90\%) & 68.4592 $\pm$ 0.0187 & 87.0517 $\pm$ 0.0334 \\
\emph{GREAT dp} &
86.5967 $\pm$ 0.0216~~(00/32 @ 90\%) & 62.9828 $\pm$ 0.0245 & 80.5810 $\pm$ 0.0192 \\
\emph{GGNN dp} &
76.8530 $\pm$ 0.0388~~(00/32 @ 90\%) & 22.9947 $\pm$ 0.0083 & 28.9142 $\pm$ 0.0520 \\
\emph{NRI encoder} &
81.7358 $\pm$ 0.0347~~(00/32 @ 90\%) & 69.0823 $\pm$ 0.0216 & 88.8749 $\pm$ 0.0452 \\
\emph{RL ablation} &
91.6981 $\pm$ 0.0122~~(03/32 @ 90\%) & 91.1424 $\pm$ 0.0120 & 92.2917 $\pm$ 0.0215 \\
\emph{GFSA Layer (ours)} &
\textbf{99.9944 $\pm$ 0.0002}~~(29/32 @ 90\%) & \textbf{99.9890 $\pm$ 0.0003} & \textbf{99.9971 $\pm$ 0.0004} \\

\bottomrule
\end{tabular}%
\hspace{0.2\textwidth}

\vspace{1em}

\begin{tabular}{rrrr}
\toprule
\textbf{Task} & \multicolumn{3}{c}{Last Read}  \\
\cmidrule(r{4pt}){2-4}
\textbf{Example size} &
\multicolumn{1}{c}{1x}&
\multicolumn{1}{c}{2x}&
\multicolumn{1}{c}{0.5x}\\
\midrule
\multicolumn{4}{c}{\textbf{100,000 training examples}}\\
\midrule

\emph{RAT nw} &
99.8602 $\pm$ 0.0020~~(11/32 @ 90\%) & 96.2865 $\pm$ 0.0083 & \textbf{99.9785 $\pm$ 0.0008} \\
\emph{GREAT nw} &
99.9099 $\pm$ 0.0015~~(14/32 @ 90\%) & 95.1157 $\pm$ 0.0100 & \textbf{99.9801 $\pm$ 0.0006} \\
\emph{GGNN nw} &
95.5197 $\pm$ 0.0121~~(04/32 @ 90\%) & 9.2216 $\pm$ 0.0658 & 86.2371 $\pm$ 0.0310 \\
\emph{RAT dp} &
99.9579 $\pm$ 0.0011~~(18/32 @ 90\%) & 42.5754 $\pm$ 0.0139 & 91.9595 $\pm$ 0.0325 \\
\emph{GREAT dp} &
\textbf{99.9869 $\pm$ 0.0005}~~(18/32 @ 90\%) & 47.0747 $\pm$ 0.0193 & 99.7819 $\pm$ 0.0028 \\
\emph{GGNN dp} &
98.4356 $\pm$ 0.0063~~(05/32 @ 90\%) & 0.9925 $\pm$ 0.0004 & 63.7686 $\pm$ 0.0940 \\
\emph{NRI encoder} &
99.8306 $\pm$ 0.0024~~(14/32 @ 90\%) & 43.4380 $\pm$ 0.0220 & 99.3851 $\pm$ 0.0051 \\
\emph{RL ablation} &
96.6928 $\pm$ 0.0131~~(02/32 @ 90\%) & 94.8530 $\pm$ 0.0164 & 97.8541 $\pm$ 0.0091 \\
\emph{GFSA Layer (ours)} &
99.6561 $\pm$ 0.0030~~(25/32 @ 90\%) & \textbf{98.9355 $\pm$ 0.0056} & 99.8973 $\pm$ 0.0020 \\

\midrule
\multicolumn{4}{c}{\textbf{100 training examples}}\\
\midrule

\emph{RAT nw} &
80.2832 $\pm$ 0.0257~~(00/32 @ 90\%) & 1.1217 $\pm$ 0.0021 & 83.4938 $\pm$ 0.0284 \\
\emph{GREAT nw} &
78.8755 $\pm$ 0.0220~~(00/32 @ 90\%) & 6.9583 $\pm$ 0.0157 & 60.9003 $\pm$ 0.0375 \\
\emph{GGNN nw} &
79.3594 $\pm$ 0.0350~~(00/32 @ 90\%) & 28.2760 $\pm$ 0.3023 & 5.6617 $\pm$ 0.0095 \\
\emph{RAT dp} &
59.5289 $\pm$ 0.0174~~(00/32 @ 90\%) & 28.9121 $\pm$ 0.0076 & 62.2680 $\pm$ 0.0500 \\
\emph{GREAT dp} &
57.0199 $\pm$ 0.0378~~(00/32 @ 90\%) & 27.1285 $\pm$ 0.0161 & 64.4819 $\pm$ 0.0339 \\
\emph{GGNN dp} &
44.3653 $\pm$ 0.0182~~(00/32 @ 90\%) & 9.6449 $\pm$ 0.0060 & 38.3370 $\pm$ 0.0223 \\
\emph{NRI encoder} &
68.6947 $\pm$ 0.0390~~(00/32 @ 90\%) & 26.6422 $\pm$ 0.0172 & 73.5216 $\pm$ 0.0312 \\
\emph{RL ablation} &
98.4823 $\pm$ 0.0087~~(06/32 @ 90\%) & 97.0341 $\pm$ 0.0141 & 99.1689 $\pm$ 0.0089 \\
\emph{GFSA Layer (ours)} &
\textbf{98.8141 $\pm$ 0.0069}~~(25/32 @ 90\%) & \textbf{97.8198 $\pm$ 0.0079} & \textbf{99.2172 $\pm$ 0.0048} \\

\bottomrule
\end{tabular}

\vspace{1em}

\hspace{0.2\textwidth}%
\begin{tabular}{rrrr}
\toprule
\textbf{Task} & \multicolumn{3}{c}{Last Write}  \\
\cmidrule(r{4pt}){2-4}
\textbf{Example size} &
\multicolumn{1}{c}{1x}&
\multicolumn{1}{c}{2x}&
\multicolumn{1}{c}{0.5x}\\
\midrule
\multicolumn{4}{c}{\textbf{100,000 training examples}}\\
\midrule

\emph{RAT nw} &
99.8333 $\pm$ 0.0021~~(22/32 @ 90\%) & 94.8665 $\pm$ 0.0172 & \textbf{99.9741 $\pm$ 0.0012} \\
\emph{GREAT nw} &
99.7538 $\pm$ 0.0043~~(16/32 @ 90\%) & 93.2187 $\pm$ 0.0181 & 99.9343 $\pm$ 0.0022 \\
\emph{GGNN nw} &
98.8240 $\pm$ 0.0080~~(09/32 @ 90\%) & 40.6941 $\pm$ 0.0302 & 88.2834 $\pm$ 0.0281 \\
\emph{RAT dp} &
99.9815 $\pm$ 0.0006~~(19/32 @ 90\%) & 68.9617 $\pm$ 0.0169 & 99.7626 $\pm$ 0.0045 \\
\emph{GREAT dp} &
\textbf{99.9868 $\pm$ 0.0007}~~(18/32 @ 90\%) & 68.4564 $\pm$ 0.0188 & 99.8809 $\pm$ 0.0029 \\
\emph{GGNN dp} &
99.3488 $\pm$ 0.0040~~(13/32 @ 90\%) & 38.3976 $\pm$ 0.0772 & 94.5246 $\pm$ 0.0576 \\
\emph{NRI encoder} &
99.8710 $\pm$ 0.0019~~(24/32 @ 90\%) & 52.7272 $\pm$ 0.0226 & 99.8390 $\pm$ 0.0058 \\
\emph{RL ablation} &
98.0828 $\pm$ 0.0109~~(03/32 @ 90\%) & 96.6400 $\pm$ 0.0185 & 98.9277 $\pm$ 0.0076 \\
\emph{GFSA Layer (ours)} &
99.4653 $\pm$ 0.0040~~(25/32 @ 90\%) & \textbf{98.7259 $\pm$ 0.0111} & 99.7763 $\pm$ 0.0033 \\

\midrule
\multicolumn{4}{c}{\textbf{100 training examples}}\\
\midrule

\emph{RAT nw} &
79.2705 $\pm$ 0.0212~~(00/32 @ 90\%) & 8.9069 $\pm$ 0.0165 & 83.7914 $\pm$ 0.0379 \\
\emph{GREAT nw} &
80.1879 $\pm$ 0.0273~~(00/32 @ 90\%) & 40.2206 $\pm$ 0.0386 & 84.5417 $\pm$ 0.0312 \\
\emph{GGNN nw} &
91.1302 $\pm$ 0.0196~~(01/32 @ 90\%) & 71.6216 $\pm$ 0.0163 & 91.7911 $\pm$ 0.0272 \\
\emph{RAT dp} &
75.9944 $\pm$ 0.0352~~(00/32 @ 90\%) & 48.0974 $\pm$ 0.0331 & 81.6254 $\pm$ 0.0312 \\
\emph{GREAT dp} &
73.6926 $\pm$ 0.0391~~(00/32 @ 90\%) & 46.2676 $\pm$ 0.0334 & 80.0267 $\pm$ 0.0511 \\
\emph{GGNN dp} &
53.8178 $\pm$ 0.0282~~(00/32 @ 90\%) & 17.8435 $\pm$ 0.0101 & 55.0784 $\pm$ 0.0481 \\
\emph{NRI encoder} &
65.3841 $\pm$ 0.0498~~(00/32 @ 90\%) & 36.4278 $\pm$ 0.0301 & 73.8556 $\pm$ 0.0106 \\
\emph{RL ablation} &
98.3220 $\pm$ 0.0098~~(06/32 @ 90\%) & \textbf{96.9613 $\pm$ 0.0150} & 99.0671 $\pm$ 0.0074 \\
\emph{GFSA Layer (ours)} &
\textbf{98.7144 $\pm$ 0.0072}~~(24/32 @ 90\%) & \textbf{96.9758 $\pm$ 0.0120} & \textbf{99.5543 $\pm$ 0.0068} \\

\bottomrule
\end{tabular}
\end{table}

\begin{figure}[hp]
    \centering
    \includegraphics[
    width=.9\textwidth,trim={5cm 6cm 5cm 6cm},clip
    ]{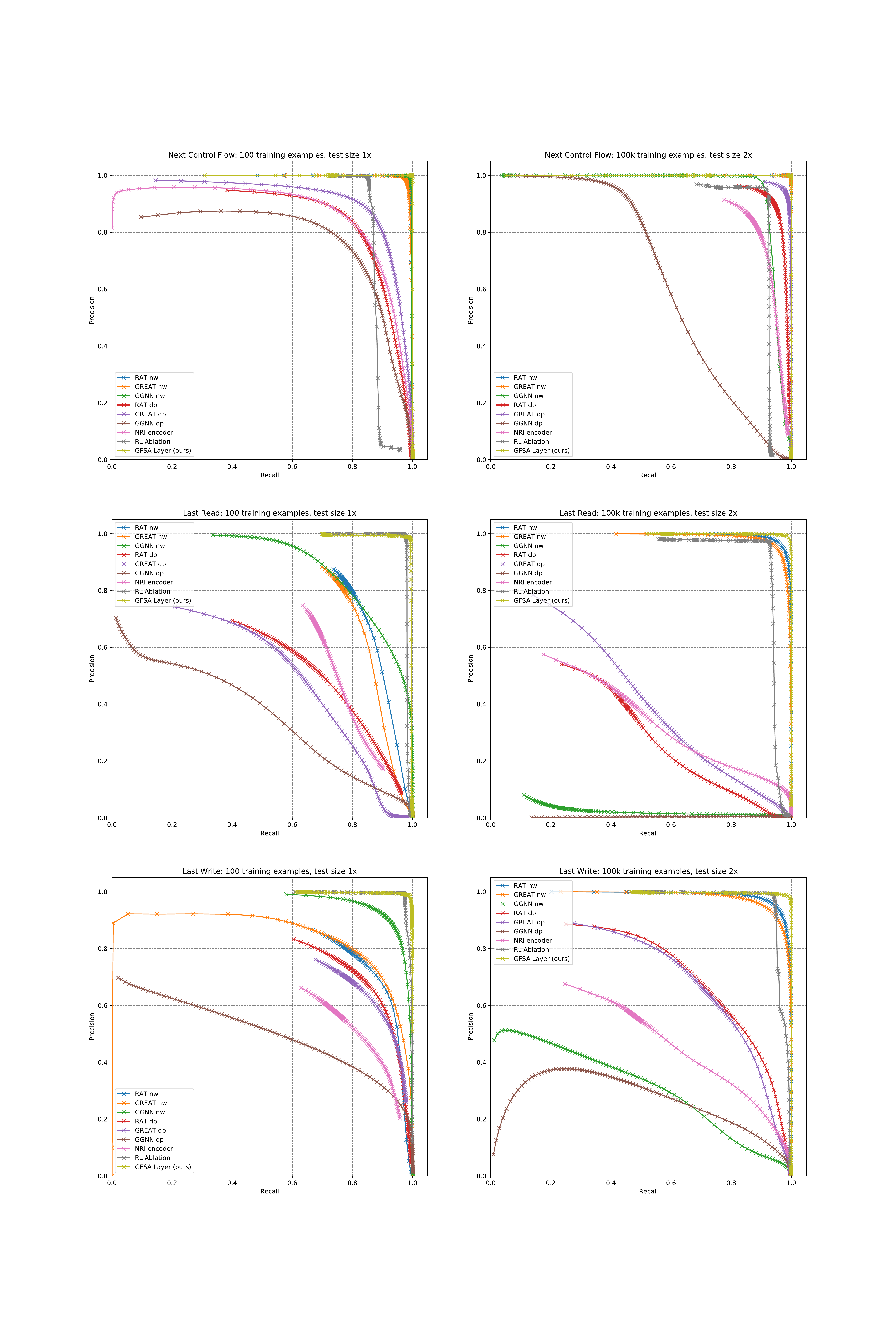}
    \caption{Precision-recall curves for a subset of the static analysis experiments that reveals interesting differences in performance: training on 100 examples and evaluating on the same data distribution, and training on 100,000 examples but evaluating on examples of twice the size. Crosshatches indicate candidate thresholds that were evaluated at test time. Best viewed in color.}
    \label{appendix:fig:edge-classification-pr}
\end{figure}

Table \ref{appendix:tab:edge_classification} contains higher-precision results for the edge-classification tasks, along with the standard error estimates computed as above. Additionally, figure \ref{appendix:fig:edge-classification-pr} shows precision-recall curves, computed for a subset of the experiments that shows the most interesting variation in performance.

\subsection{Variable Misuse} \label{appendix:var-misuse}

\subsubsection{Dataset}
We use the dataset released by \citet{hellendoorn2020global}, which is derived from a redistributable subset of the ETH 150k Python dataset \citep{raychev2016probabilistic}.\footnote{
Original Python corpus (from \citet{raychev2016probabilistic}):  \url{https://www.sri.inf.ethz.ch/py150}\\
Redistributable subset: \url{https://github.com/google-research-datasets/eth_py150_open}\\
With synthetic errors (as released by \citet{hellendoorn2020global}):\\ \url{https://github.com/google-research-datasets/great}}
For each top-level function and class definition extracted from the original dataset, this derived dataset includes up to three modified copies introducing synthetic variable misuse errors, along with an equal number of unmodified copies. For our experiments, we do additional preprocessing to support the GFSA layer: we encode the examples as graphs, and throw out examples with more than 256 nodes or 512 node-observation tuples, which leaves us with 84.5\% of the dataset from \citet{hellendoorn2020global}.

\subsubsection{Model Architectures}
As in the edge classification task, we convert the AST nodes into graph nodes, using the same helper nodes and connectivity structure described in appendix \ref{appendix:encoding-graphs-ast}. For this task, when an AST node has multiple children, we add extra edges specifying the index of each child; this is used only by the attention model, not by the GFSA layer. In addition to node features based on the AST node type, we include features based on a bag-of-subtokens representation of each AST node. We use a 10,000-token subword encoder implemented in the \texttt{Tensor2Tensor} library by \citet{vaswani2018tensor2tensor}, pretrain it on GitHub Python code, and use it to tokenize the syntax for each AST node. We then compute node features by summing over the embedding vectors of all subtokens that appear in each node. The learned embedding vectors are of dimension 128, which we project out to 256 before using as node features.

To ensure that we can compare results across different edge types in a fair way, we fix the sizes of the base models.
For the RAT and GREAT model families, we use a hidden dimension of 256 and 8 attention heads with a per-head query and value dimension of 32.
For the GGNN model family, we use a hidden dimension of 256 and a message dimension of 128.
For all models, we use positional embeddings for node features, and edge types embedded as 64-dimensional vectors.
We embed all edge types separately in the forward and reverse directions, including both the base AST edges as well as any edges added by learned edge layers; for learned edges we compute new edge features by weighting each embedding vector by the associated edge weight. For the ``@ start'' edge types, the edges are all embedded at the same time, and for the ``@ middle'' edge types, we modify the edge features after adding the new edges and use the modified edge features for all following model layers. We compute our final outputs by performing a learned dot-product operation on our final node embeddings $Y$ and then taking a softmax transformation to obtain a distribution over node pairs:
\begin{align*}
Y &= f_\theta(X_\text{node}, X_\text{edge}),
&
Z = \softmax\big(\{ \bvec{y}_{n}^T W \bvec{y}_{n'} \}_{n, n' \in N}\big).
\end{align*}
As described in \citet{vasic2019neural}, we compute a mask that indicates the location of all local variables that could be either bug locations or repair targets (along with the sentinel no-bug location). We then set the entries of $Z_{n, n'}$ to zero for the locations not contained in the mask, and renormalize so that it sums to 1 across node pairs. Note that there is always exactly one correct bug location $n$ but there could be more than one acceptable repair location $n'$; we thus sum over all correct repair locations to compute the total probability assigned to correct bug-repair pairs, and then use the standard cross-entropy loss. 

For the GFSA edges, we use an initialization temperature of $\beta = 0.2$, and fix $|Z| = 4$. We use a single finite-state automaton policy to generate two edge types by computing the trajectories when $z_0 = 0$ as well as when $z_0 = 1$. We set $T_{\max{}} = 128$.

For the NRI head edges, we use a 3-layer MLP (with hidden sizes [32, 32] and output size 2), and take a logistic sigmoid of the outputs, interpreting it as a weighted adjacency matrix for two edge types.

For the uniform random walk edges, we learn a single halting probability $p_\text{halt} = \sigma(\theta_\text{halt})$ along with adjustment parameters $a, b \in \R$ as defined in section \ref{subsec:derived-adjacency-matrix}. The output adjacency matrix is defined similarly to the GFSA model, but with all of the policy parameters fixed to move to a random neighbor with probability $1 - p_\text{halt}$ and take the \scsym{AddEdgeAndStop} action with probability $p_\text{halt}$. For this model, we only add a single edge type.

The RL agent uses the same parameterization as the GFSA layer, but samples a single trajectory for each start node and uses it to add a single edge (or no edge) from each start node. The downstream cross-entropy loss for the classification model is used as the reward for all of these trajectories. Since simply computing this reward requires a full downstream model forward pass, we run only one rollout per example with a learned scalar reward baseline $\hat{R}$. We add an additional loss term $\alpha (R - \hat{R})^2$ so that this learned baseline approximates the expected reward, and scale the REINFORCE gradient term by a hyperparameter $\beta$.

The ``Hand-engineered edges'' baseline uses the base AST edges and adds the following edge types from \citet{allamanis2018learning} and \citet{hellendoorn2020global}:
NextControlFlow,
ComputedFrom,
FormalArgName,
LastLexicalUse,
LastRead,
LastWrite,
NextToken (connecting syntactically adjacent nodes), Calls (connecting function calls to their definitions), and
ReturnsTo (connecting return statements to the function they return from).

\subsubsection{Training and Detailed Results}\label{appendix:var-misuse-detailed-results}

For all of our models, we train using the Adam optimizer for 400,000 iterations; this is enough time for all models to converge to their final accuracy. We use a batch size of 64 examples, grouping examples of similar size to avoid excessive padding.

For each model, we randomly sample 32 hyperparameter settings for the learning rate (log-uniform in $[10^{-5}, 10^{-2}]$) and gradient clipping threshold (log-uniform in $[1, 10^4]$). For the GFSA models, we also tune $\varepsilon_\text{bt-stop}$ (log-uniform in $[0.001, 0.1]$). For the RL ablation, we tune the weight of the relative weights of different gradient terms: $\alpha$ is chosen log-uniformly in $[0.00001, 0.1]$ and $\beta$ is chosen in $[0.001, 2.0]$. Over the course of training, we take a subset of approximately 7000 validation examples and compute the top-1 accuracy of each model on this subset. We then choose the hyperparameter settings and early-stopping point with the highest accuracy.

We evaluate the selected models on the size-filtered test set, containing 818,560 examples. For each example and each model, we determine the predicted classification by determining whether 50\% or more probability is assigned to the no-bug location. For incorrect examples, we then find the pair of predicted bug location and repair identifier with the highest probability (summing over all locations for each candidate repair identifier), and check whether the bug location and replacement identifier are correct. Note that if the model assigns >50\% probability to the no-bug location, but still ranks the true bug and replacement highest with the remaining probability mass, we count that as an incorrect classification but a correct localization and repair.

To compute standard error estimates, we assume that predictions are independent across different functions, but may be correlated across modified copies of the same function; we thus estimate standard error by using the analytic variance for a binomial distribution, adjusted by a factor of 3 (for buggy or non-buggy examples analyzed separately) or 6 (for averages across all examples) to account for the multiple copies of each function in the dataset. Table \ref{appendix:tab:var_misuse} contains higher-precision results for the variable misuse tasks, along with standard error estimates,  a breakdown of marginal localization and repair scores (examples where the model gets one of the locations correct but possibly the other incorrect), and an overall accuracy score capturing classification, localization, and repair.

\begin{table}[p]
\centering
\caption{Full-precision results on variable misuse task, with additional breakdown of accuracy for buggy examples. Expressed as accuracy (in \%) $\pm$ standard error.}
\label{appendix:tab:var_misuse}
\scriptsize

\begin{tabular}{lcccc}
\toprule
Example type: & \multicolumn{2}{c}{All} & No bug & With bug\\
\cmidrule(r{4pt}){2-3}
\cmidrule(r{4pt}){4-4}
\cmidrule(r{4pt}){5-5}
& Classification & Class \& Loc \& Rep & Classification& Classification  \\
\midrule
\textbf{RAT}\\
\emph{Base AST graph only}
 & {92.540 $\pm$ 0.071} & {88.225 $\pm$ 0.087} & {92.051 $\pm$ 0.073} & {93.030 $\pm$ 0.069}\\
\emph{Base AST graph, +2 layers}
 & {92.245 $\pm$ 0.072} & {87.846 $\pm$ 0.088} & {92.455 $\pm$ 0.072} & {92.035 $\pm$ 0.073}\\
\emph{Hand-engineered edges}
 & {92.704 $\pm$ 0.070} & {88.496 $\pm$ 0.086} & {92.932 $\pm$ 0.069} & {92.477 $\pm$ 0.071}\\
\emph{NRI head @ start}
 & {92.880 $\pm$ 0.070} & {88.710 $\pm$ 0.085} & {92.551 $\pm$ 0.071} & {93.208 $\pm$ 0.068}\\
\emph{NRI head @ middle}
 & {92.572 $\pm$ 0.071} & {88.423 $\pm$ 0.086} & {92.834 $\pm$ 0.070} & {92.310 $\pm$ 0.072}\\
\emph{Random walk @ start}
 & {92.997 $\pm$ 0.069} & {88.907 $\pm$ 0.084} & \textbf{93.224 $\pm$ 0.068} & {92.770 $\pm$ 0.070}\\
\emph{RL ablation @ middle}
 & {92.036 $\pm$ 0.073} & {87.278 $\pm$ 0.090} & {90.361 $\pm$ 0.080} & {93.711 $\pm$ 0.066}\\
\emph{GFSA layer (ours) @ start}
 & \textbf{93.328 $\pm$ 0.068} & \textbf{89.472 $\pm$ 0.083} & \textbf{93.101 $\pm$ 0.069} & {93.555 $\pm$ 0.066}\\
\emph{GFSA layer (ours) @ middle}
 & \textbf{93.456 $\pm$ 0.067} & \textbf{89.627 $\pm$ 0.082} & {92.662 $\pm$ 0.071} & \textbf{94.250 $\pm$ 0.063}\\

\midrule
\textbf{GREAT}\\
\emph{Base AST graph only}
 & {91.662 $\pm$ 0.075} & {86.906 $\pm$ 0.091} & {90.849 $\pm$ 0.078} & {92.475 $\pm$ 0.071}\\
\emph{Base AST graph, +2 layers}
 & {92.307 $\pm$ 0.072} & {87.902 $\pm$ 0.087} & \textbf{92.711 $\pm$ 0.070} & {91.903 $\pm$ 0.074}\\
\emph{Hand-engineered edges}
 & {92.287 $\pm$ 0.072} & {87.646 $\pm$ 0.088} & {92.577 $\pm$ 0.071} & {91.996 $\pm$ 0.073}\\
\emph{NRI head @ start}
 & {92.061 $\pm$ 0.073} & {87.447 $\pm$ 0.089} & {91.112 $\pm$ 0.077} & {93.009 $\pm$ 0.069}\\
\emph{NRI head @ middle}
 & {92.074 $\pm$ 0.073} & {87.552 $\pm$ 0.088} & \textbf{92.800 $\pm$ 0.070} & {91.347 $\pm$ 0.076}\\
\emph{Random walk @ start}
 & {92.644 $\pm$ 0.071} & {88.283 $\pm$ 0.087} & {91.872 $\pm$ 0.074} & \textbf{93.417 $\pm$ 0.067}\\
\emph{RL ablation @ middle}
 & {91.707 $\pm$ 0.075} & {86.939 $\pm$ 0.091} & {89.951 $\pm$ 0.081} & \textbf{93.464 $\pm$ 0.067}\\
\emph{GFSA layer (ours) @ start}
 & \textbf{92.963 $\pm$ 0.069} & \textbf{88.825 $\pm$ 0.085} & \textbf{92.872 $\pm$ 0.070} & {93.055 $\pm$ 0.069}\\
\emph{GFSA layer (ours) @ middle}
 & \textbf{93.019 $\pm$ 0.069} & \textbf{88.806 $\pm$ 0.085} & {92.427 $\pm$ 0.072} & \textbf{93.612 $\pm$ 0.066}\\

\midrule
\textbf{GGNN}\\
\emph{Base AST graph only}
 & {89.704 $\pm$ 0.082} & {83.521 $\pm$ 0.098} & \textbf{91.257 $\pm$ 0.076} & {88.152 $\pm$ 0.087}\\
\emph{Base AST graph, +2 layers}
 & {90.359 $\pm$ 0.080} & {84.383 $\pm$ 0.098} & {88.795 $\pm$ 0.085} & \textbf{91.922 $\pm$ 0.074}\\
\emph{Hand-engineered edges}
 & \textbf{90.874 $\pm$ 0.078} & \textbf{84.776 $\pm$ 0.096} & {90.187 $\pm$ 0.081} & {91.560 $\pm$ 0.075}\\
\emph{NRI head @ start}
 & {90.433 $\pm$ 0.080} & {84.473 $\pm$ 0.096} & \textbf{91.486 $\pm$ 0.076} & {89.380 $\pm$ 0.083}\\
\emph{NRI head @ middle}
 & {90.243 $\pm$ 0.080} & {84.412 $\pm$ 0.098} & {88.289 $\pm$ 0.087} & \textbf{92.198 $\pm$ 0.073}\\
\emph{Random walk @ start}
 & {90.315 $\pm$ 0.080} & {84.519 $\pm$ 0.096} & \textbf{91.351 $\pm$ 0.076} & {89.278 $\pm$ 0.084}\\
\emph{RL ablation @ middle}
 & {90.540 $\pm$ 0.079} & \textbf{84.959 $\pm$ 0.096} & {90.437 $\pm$ 0.080} & {90.643 $\pm$ 0.079}\\
\emph{GFSA layer (ours) @ start}
 & \textbf{90.939 $\pm$ 0.078} & \textbf{85.012 $\pm$ 0.096} & {90.083 $\pm$ 0.081} & {91.796 $\pm$ 0.074}\\
\emph{GFSA layer (ours) @ middle}
 & {90.394 $\pm$ 0.080} & \textbf{84.723 $\pm$ 0.096} & {90.983 $\pm$ 0.078} & {89.805 $\pm$ 0.082}\\

\bottomrule
\end{tabular}
\vspace{0.5em}

\begin{tabular}{lcccc}
\toprule
Example type: 
& \multicolumn{4}{c}{With bug}\\
\cmidrule(l{4pt}){2-5}
& Localization & Repair & Loc \& Repair & Class \& Loc \& Rep \\
\midrule
\textbf{RAT}\\
\emph{Base AST graph only}
 & {92.936 $\pm$ 0.069} & {91.892 $\pm$ 0.074} & {88.300 $\pm$ 0.087} & {84.399 $\pm$ 0.098}\\
\emph{Base AST graph, +2 layers}
 & {92.638 $\pm$ 0.071} & {91.541 $\pm$ 0.075} & {87.764 $\pm$ 0.089} & {83.238 $\pm$ 0.101}\\
\emph{Hand-engineered edges}
 & {93.100 $\pm$ 0.069} & {92.007 $\pm$ 0.073} & {88.388 $\pm$ 0.087} & {84.060 $\pm$ 0.099}\\
\emph{NRI head @ start}
 & {93.180 $\pm$ 0.068} & {92.304 $\pm$ 0.072} & {88.731 $\pm$ 0.086} & {84.869 $\pm$ 0.097}\\
\emph{NRI head @ middle}
 & {93.013 $\pm$ 0.069} & {92.176 $\pm$ 0.073} & {88.619 $\pm$ 0.086} & {84.011 $\pm$ 0.099}\\
\emph{Random walk @ start}
 & {93.227 $\pm$ 0.068} & {92.282 $\pm$ 0.072} & {88.726 $\pm$ 0.086} & {84.590 $\pm$ 0.098}\\
\emph{RL ablation @ middle}
 & {92.553 $\pm$ 0.071} & {91.606 $\pm$ 0.075} & {87.730 $\pm$ 0.089} & {84.195 $\pm$ 0.099}\\
\emph{GFSA layer (ours) @ start}
 & \textbf{93.820 $\pm$ 0.065} & \textbf{92.834 $\pm$ 0.070} & {89.577 $\pm$ 0.083} & {85.843 $\pm$ 0.094}\\
\emph{GFSA layer (ours) @ middle}
 & \textbf{94.058 $\pm$ 0.064} & \textbf{93.083 $\pm$ 0.069} & \textbf{89.932 $\pm$ 0.081} & \textbf{86.593 $\pm$ 0.092}\\

\midrule
\textbf{GREAT}\\
\emph{Base AST graph only}
 & {92.030 $\pm$ 0.073} & {91.156 $\pm$ 0.077} & {87.179 $\pm$ 0.091} & {82.964 $\pm$ 0.102}\\
\emph{Base AST graph, +2 layers}
 & {92.585 $\pm$ 0.071} & {91.477 $\pm$ 0.076} & {87.698 $\pm$ 0.089} & {83.093 $\pm$ 0.101}\\
\emph{Hand-engineered edges}
 & {92.174 $\pm$ 0.073} & {91.287 $\pm$ 0.076} & {87.168 $\pm$ 0.091} & {82.715 $\pm$ 0.102}\\
\emph{NRI head @ start}
 & {92.446 $\pm$ 0.072} & {91.520 $\pm$ 0.075} & {87.628 $\pm$ 0.089} & {83.781 $\pm$ 0.100}\\
\emph{NRI head @ middle}
 & {92.213 $\pm$ 0.073} & {91.166 $\pm$ 0.077} & {87.258 $\pm$ 0.090} & {82.303 $\pm$ 0.103}\\
\emph{Random walk @ start}
 & {92.849 $\pm$ 0.070} & {91.949 $\pm$ 0.074} & {88.272 $\pm$ 0.087} & {84.694 $\pm$ 0.097}\\
\emph{RL ablation @ middle}
 & {92.383 $\pm$ 0.072} & {91.295 $\pm$ 0.076} & {87.486 $\pm$ 0.090} & {83.927 $\pm$ 0.099}\\
\emph{GFSA layer (ours) @ start}
 & \textbf{93.466 $\pm$ 0.067} & \textbf{92.279 $\pm$ 0.072} & \textbf{88.845 $\pm$ 0.085} & {84.779 $\pm$ 0.097}\\
\emph{GFSA layer (ours) @ middle}
 & \textbf{93.266 $\pm$ 0.068} & \textbf{92.394 $\pm$ 0.072} & \textbf{88.863 $\pm$ 0.085} & \textbf{85.186 $\pm$ 0.096}\\

\midrule
\textbf{GGNN}\\
\emph{Base AST graph only}
 & {89.243 $\pm$ 0.084} & {87.703 $\pm$ 0.089} & {81.633 $\pm$ 0.105} & {75.785 $\pm$ 0.116}\\
\emph{Base AST graph, +2 layers}
 & \textbf{90.633 $\pm$ 0.079} & {88.948 $\pm$ 0.085} & {83.969 $\pm$ 0.099} & {79.972 $\pm$ 0.108}\\
\emph{Hand-engineered edges}
 & \textbf{90.681 $\pm$ 0.079} & {88.770 $\pm$ 0.085} & {83.524 $\pm$ 0.100} & {79.365 $\pm$ 0.110}\\
\emph{NRI head @ start}
 & {89.915 $\pm$ 0.082} & {88.151 $\pm$ 0.087} & {82.731 $\pm$ 0.102} & {77.460 $\pm$ 0.113}\\
\emph{NRI head @ middle}
 & {90.352 $\pm$ 0.080} & \textbf{89.613 $\pm$ 0.083} & \textbf{84.443 $\pm$ 0.098} & \textbf{80.535 $\pm$ 0.107}\\
\emph{Random walk @ start}
 & {89.729 $\pm$ 0.082} & {88.611 $\pm$ 0.086} & {82.956 $\pm$ 0.102} & {77.688 $\pm$ 0.113}\\
\emph{RL ablation @ middle}
 & {90.560 $\pm$ 0.079} & {89.269 $\pm$ 0.084} & \textbf{84.301 $\pm$ 0.098} & {79.480 $\pm$ 0.109}\\
\emph{GFSA layer (ours) @ start}
 & \textbf{90.939 $\pm$ 0.078} & {88.960 $\pm$ 0.085} & {83.909 $\pm$ 0.099} & {79.942 $\pm$ 0.108}\\
\emph{GFSA layer (ours) @ middle}
 & {90.217 $\pm$ 0.080} & {88.886 $\pm$ 0.085} & {83.633 $\pm$ 0.100} & {78.463 $\pm$ 0.111}\\
 
\bottomrule
\end{tabular}

\end{table}

\end{document}